\documentclass[10pt,journal,cspaper,compsoc]{IEEEtran}
\ifCLASSOPTIONcompsoc
  \usepackage[nocompress]{cite}
\else
\fi
\ifCLASSINFOpdf
    \usepackage[pdftex]{graphicx}
    \DeclareGraphicsExtensions{.pdf,.jpeg,.png}
\else
    \usepackage[dvips]{graphicx}
    \DeclareGraphicsExtensions{.eps}
\fi

\usepackage[cmex10]{amsmath}

\usepackage{algorithmic}

\usepackage{url}

\usepackage{multirow}
\usepackage{algorithm}
\usepackage{color}
\usepackage{amsmath}
\usepackage{amssymb}
\usepackage{mathrsfs}
\usepackage{booktabs}
\usepackage{comment}
\usepackage{multirow}
\usepackage{makecell,rotating}
\usepackage{array}
\usepackage{setspace}
\usepackage{framed}
\usepackage{bbm}
\usepackage{bm}
\usepackage[dvipsnames, svgnames, x11names]{xcolor}
\usepackage{makecell}
\usepackage{multirow}
\usepackage{adjustbox}
\usepackage{colortbl}

\usepackage{amsthm,amsmath,amssymb}

\usepackage{mathrsfs}

\theoremstyle{plain}
\newtheorem{theorem}{Theorem}

\newtheorem{lemma}{Lemma}

\newtheorem{definition}{Definition}
\newtheorem{assumption}{Assumption}
\newtheorem{remark}{Remark}

\definecolor{sol_light_blue}{RGB}{38, 139, 210}
\definecolor{sol_blue}{RGB}{38, 139, 210}
\definecolor{nord_blue}{RGB}{38, 139, 210}
\definecolor{sol_green}{RGB}{163, 190, 140}
\definecolor{sol_red}{RGB}{220, 50, 47}
\definecolor{nord_red}{RGB}{250, 190, 192}
\definecolor{nord_green}{RGB}{163, 190, 140}

\definecolor{beer_orange}{RGB}{242, 142, 28}

\definecolor{nordblack}{RGB}{46, 52, 64}
\definecolor{nordred}{RGB}{191, 97, 106}
\definecolor{magenta}{RGB}{215, 10, 185}
\definecolor{nordgreen}{RGB}{163, 190, 140}
\definecolor{nordblue}{RGB}{94, 129, 172}
\definecolor{nordpurple}{RGB}{180, 142, 160}

\definecolor{realtabgreen}{RGB}{44, 160, 44}
\definecolor{realtabpurple}{RGB}{148, 103, 189}
\definecolor{realtaborange}{RGB}{255, 127, 14}

\begin{document}
%
\title{Self-Supervised Multi-Frame \\Neural Scene Flow}

\author{Dongrui Liu, Daqi Liu, Xueqian Li, Sihao Lin, Hongwei xie, Bing Wang, Xiaojun Chang, and Lei Chu
\thanks{Dongrui Liu is with Shanghai Jiao Tong University. Daqi Liu and Xueqian Li are with the University of Adelaide. Sihao Lin is with RMIT University. Xiaojun Chang is with ReLER lab, AAII, University of Technology Sydney, Australia. Lei Chu is with the University of Southern California,}
}

\markboth{Journal}%
{Shell \MakeLowercase{\textit{et al.}}: Bare Demo of IEEEtran.cls for Computer Society Journals}


\IEEEcompsoctitleabstractindextext{%

\begin{abstract}
Neural Scene Flow Prior (NSFP) and Fast Neural Scene Flow (FNSF) have shown remarkable adaptability in the context of large out-of-distribution autonomous driving. Despite their success, the underlying reasons for their astonishing generalization capabilities remain unclear. Our research addresses this gap by examining the generalization capabilities of NSFP through the lens of uniform stability, revealing that its performance is inversely proportional to the number of input point clouds. This finding sheds light on NSFP's effectiveness in handling large-scale point cloud scene flow estimation tasks. Motivated by such theoretical insights, we further explore the improvement of scene flow estimation by leveraging historical point clouds across multiple frames, which inherently increases the number of point clouds. Consequently, we propose a simple and effective method for multi-frame point cloud scene flow estimation, along with a theoretical evaluation of its generalization abilities. Our analysis confirms that the proposed method maintains a limited generalization error, suggesting that adding multiple frames to the scene flow optimization process does not detract from its generalizability. Extensive experimental results on large-scale autonomous driving Waymo Open and Argoverse lidar datasets demonstrate that the proposed method achieves state-of-the-art performance.
\end{abstract}

}

\maketitle
\IEEEdisplaynotcompsoctitleabstractindextext
\IEEEpeerreviewmaketitle

\section{Introduction}

Understanding the 3D world is crucial for the advancement of various critical applications such as autonomous driving \cite{chitta2022transfuser, he2023fear, singh2022road} and robotics \cite{li2023textslam, engel2017direct, wang2022efficient}. In the fields of computer vision and autonomous driving, scene flow estimation stands out as a key endeavor, aiming to determine motion fields within dynamic environments \cite{wang20233d, teed2020raft, menze2015object, ma2019deep, schuster2021deep, maurer2018proflow}. Historically, the analysis of scene flow has predominantly relied on RGB images \cite{teed2020raft, menze2015object, schuster2021deep, maurer2018proflow, jiang2019sense}. However, with the increasing availability of 3D point cloud data, there has been a surge in research efforts to directly estimate scene flow from point clouds \cite{liu2019flownet3d, liu2019meteornet, wang2022matters, zhang2023gmsf, peng2023delflow}.

Recently, the NSFP algorithm, as proposed by Li et al. (2021), has demonstrated its strong capability to handle dense point clouds, containing upwards of 150,000 points, showcasing remarkable generalization capabilities in open-world perception scenarios \cite{najibi2022motion}, which poses significant challenges for existing learning-based approaches \cite{liu2019flownet3d, liu2019meteornet, zhang2023gmsf, peng2023delflow}. In addition, the FNSF, introduced by Li et al. (2023), employs a distance transform strategy \cite{rosenfeld1966sequential, breu1995linear} to greatly significantly accelerate the optimization speed of NSFP and maintain the state-of-the-art performance on out-of-distribution (OOD) autonomous driving scenes. Thus, NSFP and FNSF emerge as potentially powerful and dependable methods for estimating dense scene flow from two consecutive frames of point clouds in the realm of autonomous driving. Despite these advancements, the reasons behind the exceptional performance of NSFP and FNSF in processing dense or large-scale point clouds have yet to be elucidated through theoretical analysis and still remain an intuition or empirical finding. The lack of a deeper understanding of NSFP hinders further progress in the field of neural scene flow estimation.

To address this issue, we conduct a theoretical investigation into the generalization error of NSFP through the framework of uniform stability \cite{bartlett2002rademacher, bousquet2002stability}. Our findings reveal that the upper bound of NSFP's generalization error inversely correlates with the number of input point clouds. In simpler terms, as the number of point clouds increases, NSFP's generalization error decreases. This analysis provides a foundational understanding of why NSFP excels in managing large-scale scene flow optimization tasks. By elucidating the relationship between the number of point clouds and generalization error, we offer a compelling explanation for NSFP's efficacy and reliability in handling complex scene flow estimations.


\begin{figure*}[t]
	\centering
	\includegraphics[width=0.99\textwidth]{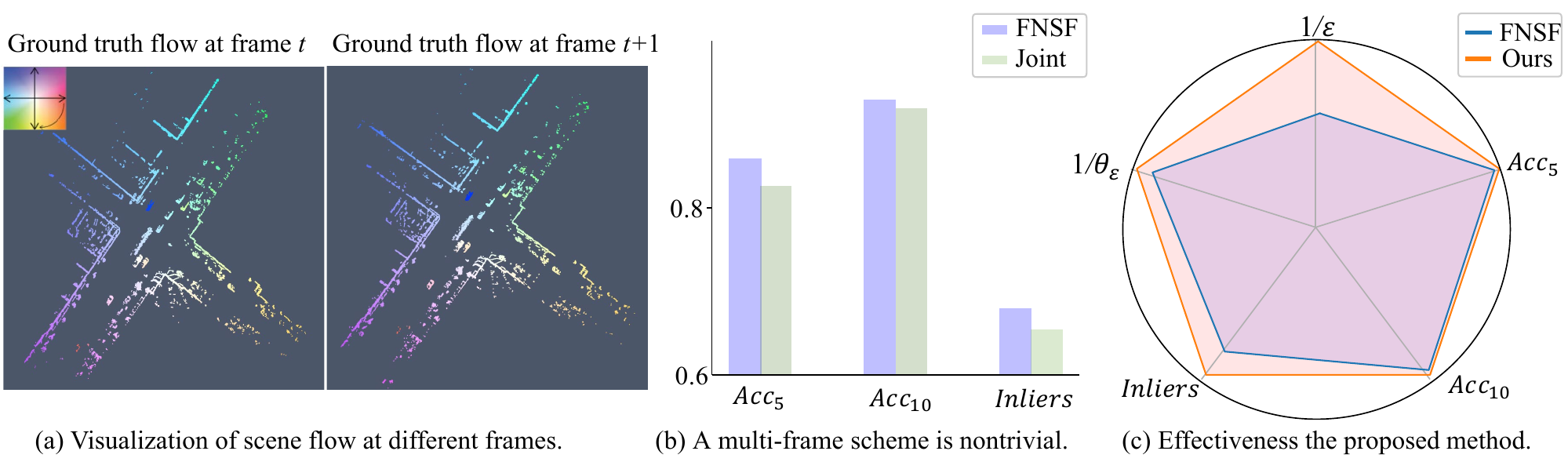}
	\caption{Current learning-based point cloud scene flow methods \cite{liu2019flownet3d, liu2019meteornet, wang2022matters, zhang2023gmsf, peng2023delflow} are trained on synthetic datasets and fail to generalize to realistic autonomous driving scenarios. Fortunately, FNSF \cite{li2023fast} shows powerful generalization ability in large lidar autonomous driving scenes. However, none of these studies exploit the useful temporal information from previous point cloud frames. Extensive studies on optical flow estimation \cite{wulff2017optical, golyanik2017multiframe, janai2018unsupervised, maurer2018proflow, liu2019selflow, stone2021smurf, hur2021self, mehl2023m} and (a) have shown that scene flow in consecutive frames are similar to each other (\emph{i.e.}, the upper left color wheel represents the flow magnitude and direction). To this end, an intuitive approach for exploiting temporal information, namely \textit{Joint}, is to force a single FNSF to jointly estimate the previous flow ($t\text{-}1 \,{\rightarrow}\, t$) and the current flow ($t \,{\rightarrow}\, t\text{+}1$). (b) shows that such an intuitive multi-frame scheme achieves worse performance than two-frame FNSF on the Waymo Open dataset. In this paper, we are the first to propose a simple and effective multi-frame point cloud scene flow estimation scheme. (c) shows that the proposed method achieves state-of-the-art on the Waymo Open dataset. For better visualization, different metrics are separately normalized. Please see Section \ref{sc:exp} for more discussions about evaluation metrics.}
	\label{fig:figure_01}
\end{figure*}


Since increasing the number of point clouds in a frame results in a better performance of NSFP, we raise an interesting question: \textit{Can we improve the scene flow estimation ($t \,{\rightarrow}\, t\text{+}1$) by using previous frames ($t\text{-}1$ and $t$), i.e., increasing the number of point clouds via adding multi-frames?} 
To this end, we seek to exploit the valuable temporal information embedded across multi-frame point clouds to improve the accuracy of two-frame scene flow estimation. Surprisingly, there appears to be a notable gap in research focused on utilizing such valuable temporal information for improving the two-frame point cloud scene flow estimations. Such a gap is particularly unexpected, because the extensive body of research in optical flow estimation \cite{wulff2017optical, janai2018unsupervised, maurer2018proflow, liu2019selflow, stone2021smurf, hur2021self, mehl2023m} have shown the importance of temporal information from previous frames, even amidst rapid motion changes in optical flow. For instance, as illustrated in Figure \ref{fig:figure_01}(a), it is evident that flows between consecutive frames bear a significant resemblance to each other, underscoring the potential benefits of integrating temporal insights into scene flow estimation for two-frame point clouds.

An intuitive solution for exploiting valuable temporal information is to force the FNSF to jointly estimate the previous flow ($t\text{-}1 \,{\rightarrow}\, t$) and the current flow ($t \,{\rightarrow}\, t\text{+}1$). In this way, temporal information can be implicitly encoded by the FNSF. However, Figure \ref{fig:figure_01}(b) shows that such an intuitive method fails to benefit from temporal information and achieves worse performance than the two-frame FNSF, \emph{i.e.,} estimating the flow from frame $t $ to frame $t\text{+}1$. 

In this study, we propose a simple and effective method for multi-frame scene flow estimation. Specifically, we employ two instances of FSNF models to calculate both the forward ($t \,{\rightarrow}\, t\text{+}1$) and backward ($t \,{\rightarrow}\, t\text{-}1$) flows. These flows, naturally opposing in direction, are then reconciled through a motion model that inverts the backward flow. In this way, the inverted backward flow and forward
flow are aligned in the same temporal direction. Finally,
we introduce a temporal fusion module to encode these
flows and predict the final flow. Figure \ref{fig:figure_01}(c) shows that the proposed method outperforms FNSF by a large margin on the Waymo Open dataset. More crucially, we theoretically analyze the generalization error of the
proposed multi-frame scene flow estimation scheme.
We derive that the generalization error of the multi-frame scheme is bounded, which guarantees the convergence
of optimization.

 To the best of our knowledge, we are the first to theoretically analyze NSFP's generalization error and explain its effectiveness in large lidar autonomous driving datasets. Based on the theoretical analysis, we exploit and use valuable temporal information to improve the scene flow estimation, as shown in Figure \ref{fig:figure_01}(c). We expect this study to provide analytical insights and encourage investigation into exploiting temporal information in scene flow estimation. Contributions of this paper can be summarized as follows:
\begin{enumerate}
	\item We present a theoretical analysis of the generalization error of NSFP, demonstrating that this error decreases as the number of point clouds increases. This insight effectively fills the gap in previous theoretical analysis and clarifies why NSFP performs outstandingly well with large-scale point clouds. 
 	\item  We propose a simple and effective strategy for
multi-frame point cloud scene flow estimation, consisting of a forward model, a backward model, a motion model, and a fusion model. We conduct a theoretical examination of the generalization error for the proposed method. The upper bound of this generalization error suggests that the inclusion of multi-frame point clouds within the optimization process does not adversely affect its generalization ability.
	\item The proposed method can be trained in a self-supervised manner and achieves state-of-the-art performance on the real-world Waymo Open and Argoverse datasets.
\end{enumerate}


\begin{table*}[t]
\caption{Notation in this paper.}
\centering
\renewcommand\arraystretch{1.25}
\begin{tabular}{c|l}
Notation & Description  \\ \hline
$S_1$ &  point cloud at time $t\text{-}1$ \\ \hline
$S_2$ & point cloud at time $t$ \\ \hline
$S_3$ &  point cloud at time $t\text{+}1$  \\ \hline
$R$ &  risk function to measure the performance of a learning algorithm \\ \hline
$D$ & point distance function to measure the distance between two point sets \\ \hline
$B$ & Bregman divergence between two functions \\ \hline
$\beta$ & uniform stability of a specific algorithm \\ \hline 
$L$ & loss function of a learning algorithm  \\ \hline 
$g_f \left(\cdot \, ;\: \mathbf{\Theta_f} \right)$& forward model to estimate the forward scene flow ($t \rightarrow t\text{+}1$) \\ \hline 
$g_b \left(\cdot \, ;\: \mathbf{\Theta_b} \right)$&  backward model to estimate the backward scene flow ($t \rightarrow t\text{-}1$) \\ \hline 
$ g_{\rm fusion} \left(\cdot \, ; \mathbf{\Theta_{\rm fusion}} \right)$ & fusion model to  estimate the fused scene flow \\ \hline 
${{\bf p}, {\bf q}} \in $ $S$ & the variables in some point clouds set $S$ \\ \hline 
${{\bf p}_i}/{{\bf q}_j}$ & $i-$/$j-$ th point cloud in the corresponding ensemble \\ \hline 
${\left| \Phi \right|}$/${\left| S \right|}$ & The size of the dataset $\Phi$/$S$ \\ \hline
$\nabla L$ & the gradient of the loss function $L$ \\ 
\hline
$\left\langle \cdot \right\rangle $ & inner product \\ 
\hline
$\mathbb{E} $ & expectation operator \\ 
\hline
\end{tabular}
\label{tab:my_label}
\end{table*}


\section{Related work}
\label{sec:related}

\subsection{3D data processing}

Using deep neural networks to process 3D data has garnered significant interest in recent years. This field primarily consists of two categories: voxel-based and point-based methodologies. Voxel-based methods usually partition the 3D space into a grid of voxels and apply standard 3D convolutions to extract features \cite{maturana2015voxnet, wu20153d}. However, the computational and memory cost of processing voxels is expensive. To this end, many approaches propose efficient data structures and convolution operations, including Kd-tree \cite{klokov2017escape}, octree \cite{riegler2017octnet}, and sparse convolution \cite{choy20194d}. On the other hand, point-based methods process point cloud data directly \cite{qi2017pointnet, qi2017pointnet2}. Furthermore, DGCNN \cite{wang2019dynamic} uses graph neural networks to encode the geometry relationship between different points. Based on these pioneer works, PointCNN \cite{li2018pointcnn}, PointConv \cite{wu2019pointconv}, and KPConv \cite{thomas2019kpconv} are further proposed to enhance industrial applications based on the point cloud, including recognition \cite{rao2022pointglr, liu2022pfmixer, liu2022robust}, detection \cite{zhou2018voxelnet, huang2023fuller}, registration \cite{liu2023self, zhu2021point, li2021pointnetlk, sarode2019pcrnet}, sampling \cite{lang2020samplenet, liu2022pointfp, chen2022point}, generation \cite{chen2021genecgan, lyu2023controllable}, and interpretation \cite{shen2021interpreting}.

\subsection{Scene flow estimation}
Scene flow tasks aim to estimate motion fields from dynamic scenes (typically two different frames). Scene flow estimation from 2D images has been extensively explored in recent years~\cite{teed2020raft, menze2015object, ma2019deep, schuster2021deep, maurer2018proflow, hur2021self, jiang2019sense}. On the other hand, researchers estimate scene flow directly from 3D point clouds via full/self-supervised training schemes \cite{liu2019flownet3d, liu2019meteornet, gu2019hplflownet, wang2020flownet3d++, puy20flot, kittenplon2020flowstep3d, wang2021festa, wu2020pointpwc, vedder2023zeroflow, wang2022matters, zhang2023gmsf, peng2023delflow, lang2023scoop}. Specifically, these methods mainly extract point-based features and compute correspondences between two point clouds. Based on accurate correspondences, these methods achieve superior performance on synthetic KITTI Scene Flow~\cite{menze2015object} and FlyingThings3D~\cite{mayer2016large} datasets. However, they fail to generalize to more realistic and larger autonomous driving scenarios \cite{pontes2020scene, li2021neural, najibi2022motion, dong2022exploiting, jin2022deformation, chodosh2023re}, \emph{e.g.}, Waymo Open~\cite{sun2020scalability} and Argoverse~\cite{chang2019argoverse} datasets. In comparison, NSFP \cite{li2021neural} uses a Multi-Layer Perception (MLP) to estimate the scene flow and demonstrates powerful generalization ability in large-scale autonomous driving scenarios, \emph{e.g.}, processing about 150k+ points. More recently, FNSF \cite{li2023fast} speeds up NSFP by using Distance Transform (DT) without sacrificing the performance on large-scale autonomous driving scenarios. 

However, all the above studies fail to exploit temporal information from previous frames. In this paper, we aim to exploit valuable temporal information from previous frames and focus on large-scale autonomous driving scenes.

\subsection{Multi-frame optical flow}
Extensive studies focus on using multi-frames to estimate optical flow \cite{golyanik2017multiframe, maurer2018proflow, ren2019fusion, schuster2021deep, hur2021self, mehl2023m}. Ren~\textit{et al.}~\cite{ren2019fusion} discovers that performance improvements are relatively smaller when the frame number is more than three. In this way, these studies obtain more accurate results by considering three consecutive frames, which achieves a compromise between temporal information and efficiency \cite{wulff2017optical, janai2018unsupervised, liu2019selflow, stone2021smurf}. Specifically, these methods aim to learn a motion model across different frames, because optical flow fields are temporally smooth and distributed around a low-dimensional linear subspace \cite{irani1999multi, janai2018unsupervised}. In this way, the motion model can exploit valuable information and predict the motion field of the current frame based on previous frames. Then, a fusion module combines the previous and current predictions to estimate a more accurate result in the current frame. However, these previous studies need human annotations. 

In contrast, we aim to exploit and use valuable temporal information to improve the two-frame lidar scene flow estimation in a self-supervised scheme. To this end, we propose a simple and effective fusion strategy.

\section{Approach}
\label{approach}
In this section, we first briefly introduce the background of the neural scene flow estimation. Then we introduce the proposed multi-frame point cloud scene flow estimation scheme in Section \ref{multi_frame}. 
Finally, we theoretically analyze the generalization error of both NSFP and the proposed multi-frame scheme in Section \ref{theoretical_analysis} for better readability and conciseness.

\subsection{Background}

\textbf{Two-frame point cloud scene flow optimization.} 
Let $\mathcal{S}_1$ and $\mathcal{S}_2$ denote the 3D point cloud sampled from a dynamic scene at time $t\text{-}1$ and $t$, respectively. Due to the movement and occlusion, the number of points in $\mathcal{S}_1$ and $\mathcal{S}_2$ are different and not in correspondence, \emph{i.e.}, $|\mathcal{S}_1| \neq|\mathcal{S}_2|$. To model the movement of each point, let $\mathbf{f}\;{\in}\; \mathbb{R}^{3}$ denote a translational vector (or flow vector) of a 3D point $\mathbf{p} \in \mathcal{S}_1$ moving from time $t\text{-}1$ to time $t$, \emph{i.e.}, $\mathbf{p}' = \mathbf{p}+\mathbf{f}$. In this way, the scene flow $\mathcal{F}_1=\{\mathbf{f}_i\}_{i=1}^{|\mathcal{S}_1|}$ is the set of translational vectors for all 3D points in $\mathcal{S}_1$.

Therefore, the optimal scene flow $\mathcal {F}^{*}$ obtains the minimal distance between the two point clouds, $\mathcal{S}_1$ and $\mathcal{S}_2$. Due to the non-rigidity motion field of the dynamic scene, the optimization of the scene flow is inherently unconstrained. To this end, a regularization term $\mbox{C}$ is usually used to constrain the motion field, \emph{e.g.}, Laplacian regularizer \cite{pontes2020scene, zeng20193d}. In this way, the optimization of the scene flow is formulated as follows

\begin{align}
    \mathcal{F}^{*} = \arg \min_{\mathcal{F}_1} \sum_{\mathbf{p} \in \mathcal{S}_1} \mbox{D} \left( \mathbf{p}+\mathbf{f}, \mathcal{S}_2 \right) + \lambda \mbox{C}, \label{eq:optim_02}
\end{align}

\noindent where $\mbox{D}$ is a point distance function, \emph{e.g.}, Chamfer distance \cite{fan2017point}. $\lambda$ is a the coefficient for the regularization term $\mbox{C}$.

\begin{figure*}[t]
	\centering
	\includegraphics[width=0.99\textwidth]{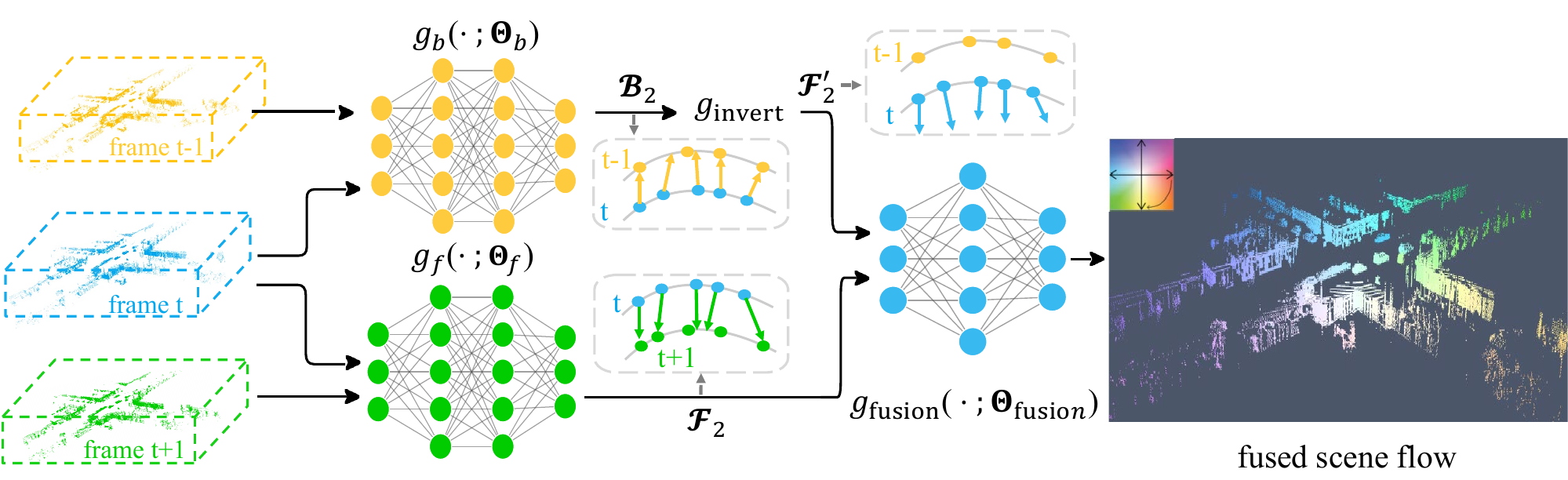}
	\caption{\textbf{Overview of the proposed multi-frame point cloud scene flow estimation scheme.} Given three consecutive frames ($t\text{-}1$, $t$, and $t\text{+}1$), we aim to estimate the scene flow from frame $t$ to frame $t\text{+}1$. Specifically, we use two models $g_f \left(\cdot \, ;\: \mathbf{\Theta_f} \right)$ and $g_b \left(\cdot \, ;\: \mathbf{\Theta_b} \right)$ to predict the forward scene flow $\mathcal{F}_2$ ($t \,{\rightarrow}\, t\text{+}1$) and the backward scene flow $\mathcal{B}_2$ ($t \,{\rightarrow}\, t\text{-}1$), respectively. Furthermore, a motion inverter $g_{\rm invert}$ and a temporal fusion model $ g_{\rm fusion} \left(\cdot \, ; \mathbf{\Theta_{\rm fusion}} \right)$ are used to estimate the fused scene flow. The upper left color wheel in the fused scene flow represents the flow magnitude and direction.}
	\label{fig:fig2}
\end{figure*}

~\\
\textbf{Neural scene flow prior.}
Compared to deep learning-based methods, NSFP utilizes traditional runtime optimization to obtain the optimal weights of the neural network without any prior knowledge or human annotations. NSFP uses the structure of the neural network as an implicit regularization, instead of adding an explicit regularization term:
\begin{align}
    \mathbf{\Theta}^* = \arg \min_{\mathbf{\Theta}} \sum_{\mathbf{p} \in \mathcal{S}_1} \mbox{D} \left( \mathbf{p} + g \left(\mathbf{p}; \mathbf{\Theta} \right), \mathcal{S}_2 \right), \label{eq:optim}
\end{align}

\noindent where $\mathbf{\Theta}$ denotes the weights of the neural network $g$. $\mathbf{p}$ is the input point cloud sampled at time $t\text{-}1$, and the flow vector $\mathbf{f} \;{=}\; g \left(\mathbf{p};\: \mathbf{\Theta} \right)$ represents the output of the neural network $g$. In this way, $\mathbf{f}^{*} \;{=}\; g \left(\mathbf{p};\: \mathbf{\Theta}^{*} \right)$ denotes the optimal flow vector. NSFP implements the neural network $g$ as an MLP and uses Chamfer distance as the loss function to optimize the scene flow.

~\\
\textbf{Fast neural scene flow.}
FNSF uses a correspondence-free loss function \emph{i.e.}, DT~\cite{rosenfeld1966sequential, breu1995linear, danielsson1980euclidean}, as a proxy for the CD loss used in NSFP. In this way, FNSF significantly accelerates the optimization process of the NSFP and becomes an approximately real-time runtime optimization method. More crucially, FNSF maintains the scalability to dense point clouds (about 150k+ points) and state-of-the-art performance on large-scale lidar autonomous driving datasets.

\subsection{Multi-Frame Scene Flow Optimization}
\label{multi_frame}

In this paper, we propose a simple and effective strategy for multi-frame point cloud scene flow estimation. Figure \ref{fig:fig2} demonstrates an overview of the proposed method. Following previous multi-frame optical flow estimation methods \cite{wulff2017optical, janai2018unsupervised, liu2019selflow, stone2021smurf}, we consider three consecutive frames ($t\text{-}1$, $t$, and $t\text{+}1$) and aim to estimate the scene flow from frame $t$ to frame $t\text{+}1$. Specifically, let $\mathcal{S}_1$, $\mathcal{S}_2$, and $\mathcal{S}_3$ be three 3D point clouds sampled from a dynamic scene at time $t\text{-}1$, $t$, and $t\text{+}1$.
The number of points in each point cloud, $|\mathcal{S}_1|$, $|\mathcal{S}_2|$, and $|\mathcal{S}_3|$, are typically different and not in correspondence, \emph{i.e.}, $|\mathcal{S}_1| \neq|\mathcal{S}_2| \neq|\mathcal{S}_3|$. 

Inspired by previous findings and Figure \ref{fig:figure_01}(a) that motion fields across different frames are temporally smooth \cite{irani1999multi, janai2018unsupervised}, we aim to use motion fields in previous frames to improve the estimation of the scene flow in the current frame. An intuitive method is to use a single model to jointly estimate the previous flow ($t\text{-}1 \,{\rightarrow}\, t$) and the current flow ($t \,{\rightarrow}\, t\text{+}1$). However, a single model fails to exploit and benefit from temporal information in such a coarse and intuitive way. Please see experimental results in Table \ref{tab:mean_time_waymo_3d}.

To effectively exploit temporal information from previous frames, we propose to use two models $g_f \left(\mathbf{p};\: \mathbf{\Theta_f} \right)$ and $g_b \left(\mathbf{p};\: \mathbf{\Theta_b} \right)$ to predict the forward scene flow $\mathcal{F}_2=\{\mathbf{f}_i\}_{i=1}^{|\mathcal{S}_2|}$ ($t \,{\rightarrow}\, t\text{+}1$) and the backward scene flow $\mathcal{B}_2=\{\mathbf{b}_i\}_{i=1}^{|\mathcal{S}_2|}$ ($t \,{\rightarrow}\, t\text{-}1$), respectively. The optimization of these two models can be formulated as follows.
\begin{align}
\label{flowf}
    \mathbf{\Theta_f}^* = \arg \min_{\mathbf{\Theta_f}} \sum_{\mathbf{p} \in \mathcal{S}_2} \mbox{D} \left( \mathbf{p} + g_f \left(\mathbf{p}; \mathbf{\Theta_f} \right), \mathcal{S}_3 \right). \\
    \mathbf{\Theta_b}^* = \arg \min_{\mathbf{\Theta_b}} \sum_{\mathbf{p} \in \mathcal{S}_2} \mbox{D} \left( \mathbf{p} + g_b \left(\mathbf{p}; \mathbf{\Theta_b} \right), \mathcal{S}_1 \right).
    \label{eq:optim_main}
\end{align}

~\\
\textbf{Temporal scene flow inversion.} 
Given the forward and the backward scene flow, we aim to further exploit useful temporal information from these flows. However, useful temporal information cannot be directly extracted, because the forward and the backward flow represent the opposite motion field, \emph{i.e.}, $t \,{\rightarrow}\, t\text{+}1$ is opposite to $t \,{\rightarrow}\, t\text{-}1$. In this way, these flows conflict with each other. To this end, we introduce a motion model $g_{\rm invert} \left(\mathbf{b};\: \mathbf{\Theta_{\rm invert}} \right)$ to invert the backward flow $\mathcal{B}_2=\{\mathbf{b}_i\}_{i=1}^{|\mathcal{S}_2|}$ to the flow $\mathcal{F}_2^{'}=\{\mathbf{f}_i^{'}\}_{i=1}^{|\mathcal{S}_2|}$, which has the same direction of the forward flow. Therefore, we have $\mathbf{f}^{'} = g_{\rm invert} \left(\mathbf{b}; \mathbf{\Theta_{\rm invert}} \right)$, where ${\mathbf{b} \in \mathcal{B}_2}$.

~\\
\textbf{Temporal fusion.}
We can fuse the forward and the inverted backward scene flow and exploit useful temporal information. Specifically, we adopt a simple and effective temporal fusion model $g_{\rm fusion} \left(\mathbf{f},\mathbf{f^{'}};\: \mathbf{\Theta_{\rm fusion}} \right)$ to estimate the final scene flow, which is based on multi-frame point clouds. In this way, the fused flow can better overcome occlusions and out-of-view motion, because additional information of the occluded regions can be extracted from different frames/views \cite{maurer2018proflow, schuster2020sceneflowfields++, schuster2021deep}.
\begin{align}
\label{eq:optim_reverse}
    \mathbf{\Theta_{\rm invert}}^*, \mathbf{\Theta}_{\rm fusion}^* = 
    \arg\min_{\mathbf{\Theta_{\rm invert}}, \mathbf{\Theta}_{\rm fusion}}\\ \nonumber
    \sum_{\mathbf{p} \in \mathcal{S}_2} \mbox{D} \left(\mathbf{p} + g_{\rm fusion} \left(\mathbf{f}, \mathbf{f^{'}}; \mathbf{\Theta_{\rm fusion}} \right), \mathcal{S}_3 \right), 
\end{align}

\noindent where $\mathbf{f} = g_f \left(\mathbf{p}; \mathbf{\Theta_f} \right)$ and $\mathbf{f}^{'} = g_{\rm invert} \left(\mathbf{b}; \mathbf{\Theta_{\rm invert}} \right)$.

\subsection{Theoretical Analysis}
\label{theoretical_analysis}

Drawing on the theoretical frameworks proposed by \cite{devroye1979distribution, bartlett2002rademacher, bousquet2002stability, liu2016algorithm}, we adopt uniform stability, as introduced by \cite{bartlett2002rademacher, bousquet2002stability}, as a metric to evaluate the generalization performance of both NSFP and the method proposed in this study. We initiate by presenting the essential technical tools.

\subsubsection{Notations}
\label{usnsfp}
Let $\mathcal{X} \in \mathbb{R}$ and $\mathcal{Y}\in \mathbb{R} $ be the input and output space, we consider the training dataset  
\begin{equation}
\Phi = \left\{ {{z_1}, \cdots ,{z_{\left| \Phi \right|}}} \right\},
\label{nota1}
\end{equation}
where we have ${z_i} = \left\{ {{x_i},{y_i}} \right\}{|_{i = 1, \cdots ,\left| \Phi \right|}}$ and $\mathcal{Z}=\mathcal{X}\times \mathcal{Y}$ drawn independent and identically distributed from some unknown distribution $\Xi$.  The learning algorithm, denoted by $A$, is to learn some function from $\mathcal{Z}^{\left| \Phi \right|}$ into $\mathcal{F}\subset \mathcal{Y}^\mathcal{X}$, mapping the dataset $\Phi$ onto the function $A_\Phi$ from $\mathcal{X}$ to $\mathcal{Y}$. Since we are considering a neural network-based algorithm, $A$ here is related to the learnable neural network parameters. 
We use  $\mathbb{E}_z$ to represent the expectation operator. Given a training dataset $\Phi$, we also consider a modified version by replacing the $i$-th element by a new sample $z_m^{'}$, yielding 
\begin{equation}
{\Phi^m} = \left\{ {{z_1}, \cdots ,{z_{m - 1}},z_m^{'},,{z_{m - 1}}, \cdots ,{z_{\left| \Phi \right|}}} \right\}.
\label{nota2}
\end{equation}
We assume the replacement example $z_m^{'}$ is drawn from $\Xi$ and is independent of $\Phi$.
We use the $risk$ (also known as $generalization \ error$) to measure the performance of a learning algorithm \cite{bartlett2002rademacher, bousquet2002stability}, which can be denoted by 
\begin{equation}
R\left( {A,\Phi} \right) = {\mathbb{E}_z}\left[ {\ell\left( {{A_{\Phi}}, z} \right)} \right],
\label{nota3}
\end{equation}
where $\ell$ represents the loss function of a learning algorithm.  
The classical estimator for the $risk$ of the dataset ${\Phi^m}$ is the $resubstitution \ estimate$ (also known as $empirical \ error$)\cite{bousquet2002stability}, defined as
\begin{equation}
R\left( {A,{\Phi^m}} \right) = \frac{1}{{\left| \Phi \right|}}\sum\limits_{i = 1}^{\left| \Phi \right|} {\ell \left( {{A_{{\Phi^m}}}, z_i} \right)}. 
\label{nota4}
\end{equation}

\subsubsection{Assumptions and Main Tools}

The objective of this study is to establish bounds on the disparity between empirical and generalization errors for particular algorithms, which can be defined in the following.

\begin{definition}
 Given some algorithm $A$, its uniform stability $\beta$ exists with respect to (w.r.t.) its loss function $\ell$ if the flowing holds
 \begin{align}
 \forall \Phi \in Z,\forall m \in \left\{ {1, \cdots ,\left| \Phi \right|} \right\}, \nonumber \\
 \Delta R  \buildrel \Delta \over =  \left| {R \left( {{A},\Phi} \right) -  R \left( {{A},\Phi^m} \right)} \right| \le \beta.
 \label{mato1}
\end{align}
\end{definition}

To bound the uniform stability, we need some probability measure, such as the Bregman divergence \cite{mohri2018foundations}, which is defined by

\begin{definition}
{\bf Bregman divergence:} Let $L: \mathcal{H} \rightarrow \mathbb{R}$ be a strictly convex function that is continuously differentiable on int $\mathcal{H}$. For all distinct $g,h \in \mathcal{H} $, then the Bregman divergence is defined as
\begin{equation}
\label{bregman1}
{B_L}\left( {g||h} \right) = L\left( g \right) - L\left( h \right) - \left\langle {g - h,\nabla L\left( h \right)} \right\rangle 
\end{equation}
\end{definition}

Some key properties of Bregman divergence \cite{mohri2018foundations} are given in the following: 

\begin{lemma}
\label{bregman2}
Bregman divergence is non-negative and additive. For example, give some convex functions $F_1$, $F_2$ and $F=F_1+F_2$, for any $g,h\in \mathcal{H}$, we have 
\begin{equation}
 {B_F}\left( {g||h} \right) = {B_{{F_1}}}\left( {g||h} \right) + {B_{{F_2}}}\left( {g||h} \right)   
\end{equation}
and 
\begin{equation}
{B_F}\left( {g||h} \right) \ge 0.  
\end{equation}

\end{lemma}

To get the theoretical results, we need some mild assumptions for the statistics of the point clouds and the related neural networks. The interested readers are referred to the works \cite{devroye1979distribution, bousquet2002stability, zhang2002covering, liu2016algorithm} for more applications of the related assumptions. 

\begin{assumption}
\label{asum1}
The point clouds ${\bf P} \in S_2$, ${\bf Q} \in S_3$, and ${\bf R} \in S_3$ contains a finite points and vector spaces of point clouds and neural network (${\bf{\Theta }}$) are bounded,
\begin{align}
{\left| {{S_i}} \right|_{i = 1,2,3}} < \infty ,{\left\| {\bf{P}} \right\|_F} \le {\sigma _P}, \nonumber  \\ {\left\| {\bf{Q}} \right\|_F} \le {\sigma _Q},
{\left\| {\bf{R}} \right\|_F} \le {\sigma _R},{\left\| {\bf{\Theta }} \right\|_F} \le {\sigma _{\bf{\Theta }}}.
\label{as1}
\end{align}
\end{assumption}
In this assumption, we bound the norm of point clouds and related neural networks (forward model), which is reasonable and achievable in practice for point clouds without outliers (substantial value).

To enable the downstream analysis without loss of generality, we assume the minimum of the summation operators are given by 
\begin{equation}
{{{\bf{\hat x}}}_k} = \mathop {\arg \min }\limits_{{\bf{x}} \in {S_3}} \left\| {{\bf{p}} - {\bf{x}}} \right\|_2^2
\label{ea4}
\end{equation}
and
\begin{equation}
{{{\bf{\hat p}}}_l} = \mathop {\arg \min }\limits_{{\bf{y}} \in {S_2}} \left\| {{\bf{q}} - {\bf{y}}} \right\|_2^2 = \mathop {\arg \min }\limits_{{\bf{p}} \in {S_2}} \left\| {{\bf{q}} - \left( {{\bf{\Theta p}} + {\bf{p}}} \right)} \right\|_2^2.
\label{ea5}
\end{equation}

Let ${\bf p}_i$ and ${\bf q}_j$ be the $i$-th and $j$-th point clouds in the $S_2$ and $S_3$, respectively.  
Then, for the NSFP problem, we can rewrite the loss function in Eq. \eqref{flowf} as
\begin{equation}
L\left( {{\bf{\Theta }},{\bf{p}};{S_3}} \right) = {L_p}\left( {{\bf{\Theta }},{\bf{p}};{{{\bf{\hat x}}}_k}} \right) + {L_q}\left( {{\bf{\Theta }},{{{\bf{\hat p}}}_l};{{\bf{q}}_j}} \right),
\label{ea22}
\end{equation}
where 
\[{L_p}\left( {{\bf{\Theta }},{\bf{p}};{{{\bf{\hat x}}}_k}} \right) = \frac{1}{{\left| {{S_2}} \right|}}\sum\limits_{i = 1}^{\left| {{S_2}} \right|} {\left\| {{\bf{\Theta }}{{\bf{p}}_i} + {{\bf{p}}_i} - {{{\bf{\hat x}}}_k}} \right\|_2^2},\]
and
\[{L_q}\left( {{\bf{\Theta }},{{{\bf{\hat p}}}_l};{{\bf{q}}}} \right) = \frac{1}{{\left| {{S_3}} \right|}}\sum\limits_{j = 1}^{\left| {{S_3}} \right|} {\left\| {\left( {{\bf{\Theta }}{{{\bf{\hat p}}}_l} + {{{\bf{\hat p}}}_l}} \right) - {{\bf{q}}_j}} \right\|_2^2} \]

We include the following mild assumptions for the loss functions $L_p$ and $L_q$:
\begin{assumption}
\label{asum2}
For some ${{\sigma _p}}$, for any  ${\bf{\Theta }},{{\bf{\Theta }}_m} \in \Theta$, the loss function ${L_p}$ is bounded by
\begin{equation}
\left| {{L_p}\left( {{\bf{\Theta }},{\bf{p}};{{{\bf{\hat x}}}_k}} \right) - {L_p}\left( {{{\bf{\Theta }}_m},{\bf{p}};{{{\bf{\hat x}}}_k}} \right)} \right| \le {\sigma _P} {\left\| {\left( {{\bf{\Theta }} - {{\bf{\Theta }}_m}} \right){\bf{p}}} \right\|_2}.
\label{as3}
\end{equation}
\end{assumption}
For any network outputs (estimates) ${{\bf{\Theta }}{{{\bf{\hat p}}}_k} + {{{\bf{\hat p}}}_k}}$ and ${{\bf{\Theta }}{{{\bf{\hat p}}}_l} + {{{\bf{\hat p}}}_l}}$, the loss $L_q$ is ${{\sigma _{\bf{\Theta }}} + 1}$ admissible, such that
\begin{equation}
\left| {{L_q}\left( {{\bf{\Theta }},{{{\bf{\tilde p}}}_l};{{\bf{q}}_k}} \right) - {L_q}\left( {{{\bf{\Theta }}_m},{{{\bf{\hat p}}}_l};{{\bf{q}}_k}} \right)} \right| \le \left( {{\sigma _{\bf{\Theta }}} + 1} \right){\left\| {{{{\bf{\tilde p}}}_l} - {{{\bf{\hat p}}}_l}} \right\|_2}
\label{as2}
\end{equation}
Besides, ${L_q}$ is $c$ -strongly convex: 
\begin{equation}
\label{as4}
\left\langle {{{{\bf{\tilde p}}}_l} - {{{\bf{\hat p}}}_l},\nabla {L_q}\left( {.,{{{\bf{\tilde p}}}_l}} \right) - \nabla {L_q}\left( {.,{{{\bf{\hat p}}}_l}} \right)} \right\rangle  \ge c\left\| {{{{\bf{\tilde p}}}_l} - {{{\bf{\hat p}}}_l}} \right\|_2^2.
\end{equation}

\begin{assumption}
\label{asum5}
There exists a subset $\Omega  = \left\{ {{{\bf{d}}_1}, \cdots ,{{\bf{d}}_{\left| \Omega \right|}}} \right\} \subset \left\{ {{{\bf{p}}_1}, \cdots ,{{\bf{p}}_{\left| {{S_2}} \right|}}} \right\}$ such that for any point cloud ${\bf{p}}$ in considered tasks, ${\bf{p}}$ can be reconstructed with a small reconstruction error ($\left\| \eta  \right\| \le \varepsilon$): ${\bf{p}} = \sum\nolimits_{j = 1}^{\left| \Omega \right|} {{\alpha _j}{{\bf{d}}_j} + {\eta _j}}$, where $\alpha \in R$ and $\left\| \alpha  \right\| \le r $.  
\end{assumption}

The above four assumptions were used to bound the network function, and similar assumptions have been used and demonstrated effective in theoretical works \cite{zhang2002covering, liu2016algorithm}. 
We begin our demonstration by presenting an outline of the proofs for our principal theories. We start by utilizing the statistical characteristics (specifically, Bregman convergence) of selected subset point clouds, constructing these subsets from the original point clouds. Subsequently, we delve into examining the upper bounds of these subset point clouds. The pivotal findings are then derived from this theoretical analysis and subsequent calculations.

\subsubsection{Key Theorems}

Our first goal here is to upper-bound the NSFP algorithm as defined in the following:

\begin{definition}
{\bf Uniform Stability of NSFP:} An algorithm is $\beta$ uniformly stable with respect to the loss function $L$ if the following holds with high probability:
\begin{equation}
\Delta R\left( {L,\left\{ {{S_2},{S_3}} \right\}} \right) = \left| {{L}_p\left( {{\bf{\Theta }},{\bf{p}};{{{\bf{\hat x}}}_k}} \right) - {L_p}\left( {{{\bf{\Theta }}_m},{\bf{p}};{{{\bf{\hat x}}}_k}} \right)} \right| \le \beta,
\label{def1}
\end{equation}
where ${{\bf{\Theta }}_m}$ is the optimal forward models of the loss function $L$ over the datasets $S^m_2$ and $S^m_3$ in which we replace its $m$-th sample $ \left( {{{\bf{p}}_m}, {{{\bf{\hat p}}}_l}} \right)$ by a random new point cloud $\left( {{\bf{p}}_m^{'},{\bf{\hat p}}_l^{'}} \right)$. 
\end{definition}

Based on the provided definitions, certain mild assumptions, and comprehensive derivations, we obtain the following theoretical theoretical results.
\begin{theorem}
\label{thmss1}
With the above definitions and some assumptions, for some random sample in $\left\{ {{S_2},{S_3}} \right\}$, with high probability, we have, 
\begin{equation}
{\beta _{{\rm{NSFP}}}} \le \frac{{\left| \Omega  \right|{\sigma _p}}}{4}\left( {rv + \sqrt {{r^2}{v^2} + \frac{{8v{\sigma _{\bf{\Theta }}}\varepsilon }}{{\left| \Omega  \right|}}} } \right) + {\sigma _{\bf{\Theta }}}{\sigma _p}\varepsilon,
\label{thm1}
\end{equation}
where $v = \frac{{{\sigma _p}}}{{\left| {{S_2}} \right|}} + \frac{{{\sigma _{\bf{\Theta }}} + 1}}{{\left| {{S_3}} \right|}}$ and all variables except $S_2$ and $S_3$ can be considered as constants. 
\end{theorem} 

\begin{proof}
{\bf Proof sketch:} To define limits on the differences between empirical errors and generalization errors for specific algorithms, we initially explore the statistical correlation between the subset and original point clouds. This exploration enables us to ascertain an upper limit for forward model errors. Subsequently, we focus on the Bregman divergence, utilizing it as a pivotal statistical metric, from which we deduce the crucial inequality. This process culminates in the formulation of a comprehensive proof of our theorems. It's important to mention that, although our analysis is based on a linear network model, empirical evidence from case studies has shown that it performs well in both linear and nonlinear network models. 

{\bf Statistical Relationship between the Subset and Original Point Clouds:}
With Assumption \ref{asum2} and Cauchy-Schwarz inequality, we have 
\begin{equation}
\label{subsetEQ}
\begin{array}{l}
\left| {{L_p}\left( {{\bf{\Theta }},{\bf{p}};{{{\bf{\hat x}}}_k}} \right) - {L_p}\left( {{{\bf{\Theta }}_m},{\bf{p}};{{{\bf{\hat x}}}_k}} \right)} \right|\\
 \le {\sigma _p}{\left\| {\left( {{\bf{\Theta }} - {{\bf{\Theta }}_m}} \right){\bf{p}}} \right\|_2}\\
 \le \sqrt {\sum\limits_j {\alpha _j^2} } \sqrt {\sum\limits_{j = 1}^{\left| \Omega \right|} {\left\| {\left( {{\bf{\Theta }} - {{\bf{\Theta }}_m}} \right){{\bf{d}}_j}} \right\|_2^2} }  + {\left\| {\left( {{\bf{\Theta }} - {{\bf{\Theta }}_m}} \right)} \right\|_2}{\left\| \eta  \right\|_2}\\
 \le r\sqrt {\sum\limits_{j = 1}^{\left| \Omega \right|} {\left\| {\left( {{\bf{\Theta }} - {{\bf{\Theta }}_m}} \right){{\bf{d}}_j}} \right\|_2^2} }  + \frac{{2{\sigma _{\bf{\Theta }}}\varepsilon }}{{\left| {{S_2}} \right|}}
\end{array}
\end{equation}
Then our goal is to bound the ${\left\| {\left( {{\bf{\Theta }} - {{\bf{\Theta }}_m}} \right){{\bf{d}}}} \right\|}_2$, which is based on the Bregman divergence between the point clouds $\Phi$ and its subset $\Omega$.

With the definitions in Section \ref{usnsfp}, we know that the loss function $L$ and $L_m$ are defined over the original dataset $S_2$ and $S_3$. For the same loss functions defined over the subset $\Omega$, we can denote them as $L^\Omega$ and $L^\Omega_m$ for notation compactness. Considering the non-negativity and additivity of the Bregman divergence (Lemma \ref{bregman2}), we can have
\begin{equation}
\small
\label{eq23}
{B_{{L_q}}}\left( {{{\bf{\Theta }}_m}||{\bf{\Theta }}} \right) \le {B_L}\left( {{{\bf{\Theta }}_m}||{\bf{\Theta }}} \right),{B_{{L_q}}}\left( {{{\bf{\Theta }}_m}||{\bf{\Theta }}} \right) \le {B_{{L_m}}}\left( {{{\bf{\Theta }}_m}||{\bf{\Theta }}} \right)    
\end{equation}
and 
\begin{equation}
\label{eq24}
\begin{array}{l}
{B_{L_q^\Omega }}\left( {{{\bf{\Theta }}_m}||{\bf{\Theta }}} \right) + {B_{L_q^\Omega }}\left( {{\bf{\Theta }}||{{\bf{\Theta }}_m}} \right)\\
 \le \kappa \left[ {{B_{{L_q}}}\left( {{{\bf{\Theta }}_m}||{\bf{\Theta }}} \right) + {B_{{L_q}}}\left( {{\bf{\Theta }}||{{\bf{\Theta }}_m}} \right)} \right]
\end{array},
\end{equation}
for some $\kappa>0$.

{\bf Key Inequalities:} We concentrate on establishing the critical inequalities between the Bregman divergence of the initial point clouds and the divergence observed in their subsets.  
We start by showing the key inequality of ${B_{L_q^\Omega}}\left( {{{\bf{\Theta }}_m}||{\bf{\Theta }}} \right) + {B_{L_q^\Omega}}\left( {{\bf{\Theta }}||{{\bf{\Theta }}_m}} \right)$:
\begin{equation}
\small
\begin{array}{*{20}{l}}
{{B_{L_q^\Omega }}\left( {{{\bf{\Theta }}_m}||{\bf{\Theta }}} \right) + {B_{L_q^\Omega }}\left( {{\bf{\Theta }}||{{\bf{\Theta }}_m}} \right)}\\
\begin{array}{l}
 = \frac{1}{{\left| \Omega  \right|}}\sum\limits_{i = 1}^{\left| \Omega  \right|} {\left\langle {{\bf{\Theta }} - {{\bf{\Theta }}_m},\nabla {L_q}\left( {{\bf{\Theta }},{{{\bf{\hat p}}}_l};{{\bf{q}}_i}} \right){\bf{d}}_i^T} \right\rangle } \\
 - \frac{1}{{\left| \Omega  \right|}}\sum\limits_{i = 1}^{\left| \Omega  \right|} {\left\langle {{\bf{\Theta }} - {{\bf{\Theta }}_m},\nabla {L_q}\left( {{{\bf{\Theta }}_m},{{{\bf{\hat p}}}_l};{{\bf{q}}_i}} \right){\bf{d}}_i^T} \right\rangle } 
\end{array}\\
{ = \frac{1}{{\left| \Omega  \right|}}\sum\limits_{i = 1}^{\left| \Omega  \right|} {\left\langle {\left( {{\bf{\Theta }} - {{\bf{\Theta }}_m}} \right){{\bf{d}}_i},\nabla {L_q}\left( {{\bf{\Theta }},{{{\bf{\hat p}}}_l};{{\bf{q}}_i}} \right) - \nabla {L_q}\left( {{{\bf{\Theta }}_m},{{{\bf{\hat p}}}_l};{{\bf{q}}_i}} \right)} \right\rangle } }\\
{ \ge \frac{c}{{\left| \Omega  \right|}}\sum\limits_{i = 1}^{\left| \Omega  \right|} {\left\| {\left( {{\bf{\Theta }} - {{\bf{\Theta }}_m}} \right){{\bf{d}}_i}} \right\|_2^2} }
\end{array}
\end{equation}
where the inequality holds from Assumptions \ref{asum2} and results given in Eq. \eqref{eq23}. Since the mean square error is considered, we have $c=2$.

Since ${{\bf{\Theta }}_m}$ and ${\bf{\Theta }}$ are the optimal forward models of $L$ and $L_m$, we have ${\nabla _L}\left( {\bf{\Theta }} \right) = 0$ and ${\nabla _{L_m}}\left( {\bf{\Theta }}_m \right) = 0$. Then with the definition in Eq. \eqref{bregman1}, we obtain 
\begin{equation}
\label{keuieq1}
\begin{array}{l}
{B_L}\left( {{{\bf{\Theta }}_m}||{\bf{\Theta }}} \right) + {B_{{L_m}}}\left( {{\bf{\Theta }}||{{\bf{\Theta }}_m}} \right)\\
 = L\left( {{{\bf{\Theta }}_m}} \right) - L\left( {\bf{\Theta }} \right) + {L_m}\left( {\bf{\Theta }} \right) - {L_m}\left( {{{\bf{\Theta }}_m}} \right)\\
 = \left( {L\left( {{{\bf{\Theta }}_m}} \right) - {L_m}{{\bf{\Theta }}_m}} \right) + \left( {{L_m}\left( {\bf{\Theta }} \right) - L\left( {\bf{\Theta }} \right)} \right)\\
 = {\textstyle{1 \over {\left| {{S_2}} \right|}}}\left[ {{L_p}\left( {{\bf{\Theta }},{{\bf{p}}_m};{{{\bf{\hat x}}}_k}} \right) - {L_p}\left( {{{\bf{\Theta }}_m},{{\bf{p}}_m};{{{\bf{\hat x}}}_k}} \right)} \right] \\ + {\textstyle{1 \over {\left| {{S_2}} \right|}}}\left[ {{L_p}\left( {{\bf{\Theta }},{\bf{p}}_m^{'};{{{\bf{\hat x}}}_k}} \right) - {L_p}\left( {{{\bf{\Theta }}_m},{\bf{p}}_m^{'};{{{\bf{\hat x}}}_k}} \right)} \right]\\
 + {\textstyle{1 \over {\left| {{S_3}} \right|}}}\left[ {{L_q}\left( {{\bf{\Theta }},{{{\bf{\hat p}}}_l};{{\bf{q}}_i}} \right) - {L_q}\left( {{{\bf{\Theta }}_m},_l^{'}{{{\bf{\hat p}}}_l};{\bf{q}}_i^{'}} \right)} \right] \\ + {\textstyle{1 \over {\left| {{S_3}} \right|}}}\left[ {{L_q}\left( {{\bf{\Theta }},{\bf{\hat p}};{{\bf{q}}_i}} \right) - {L_q}\left( {{{\bf{\Theta }}_m},{\bf{\hat p}}_l^{'};{\bf{q}}_i^{'}} \right)} \right]
 \end{array}.
\end{equation}

Considering Eq. \eqref{eq24} and Assumptions \ref{asum1}-\ref{asum5}, we get
\begin{equation}
\small
\label{keuieq2}
\begin{array}{l}
{B_L}\left( {{{\bf{\Theta }}_m}||{\bf{\Theta }}} \right) + {B_{{L_m}}}\left( {{\bf{\Theta }}||{{\bf{\Theta }}_m}} \right)\\
\le \kappa \left( {{\textstyle{{{\sigma _p}} \over {\left| {{S_2}} \right|}}} + {\textstyle{{{\sigma _{\bf{\Theta }}} + 1} \over {\left| {{S_3}} \right|}}}} \right)\left( {{{\left\| {\left( {{\bf{\Theta }} - {{\bf{\Theta }}_m}} \right){{\bf{p}}_m}} \right\|}_2} + {{\left\| {\left( {{\bf{\Theta }} - {{\bf{\Theta }}_m}} \right){\bf{p}}_m^{'}} \right\|}_2}} \right)\\
\le \kappa \left( {{\textstyle{{{\sigma _p}} \over {\left| {{S_2}} \right|}}} + {\textstyle{{{\sigma _{\bf{\Theta }}} + 1} \over {\left| {{S_3}} \right|}}}} \right)\left( {r{{\left\| {\left( {{\bf{\Theta }} - {{\bf{\Theta }}_m}} \right){\bf{d}}} \right\|}_2} + \frac{{2{\sigma _{\bf{\Theta }}}\varepsilon }}{{\left| \Omega \right|}}} \right)
 \end{array}
\end{equation}
The last inequality in Eq. \eqref{keuieq2} holds with some mathematical manipulation of the reconstruction function shown in Assumption \ref{asum5} and the inequality shown in Eq. \eqref{subsetEQ}. 

{\bf Proof Completing:} Let $U = \sum\limits_{i = 1}^{\left| {{\Omega}} \right|} {{{\left\| {\left( {{\bf{\Theta }} - {{\bf{\Theta }}_m}} \right){{\bf{d}}_i}} \right\|}_2}}$, comparing the inequalities shown in Eq. \eqref{keuieq1} and 
 Eq. \eqref{keuieq2}, we can get 
\begin{equation}
\begin{array}{l}
\frac{2}{{\left| {{\Omega}} \right|}}\sum\limits_{i = 1}^{\left| {{\Omega}} \right|} {\left\| {\left( {{\bf{\Theta }} - {{\bf{\Theta }}_m}} \right){{\bf{d}}_i}} \right\|_2^2}  \\ \le \kappa \left( {\frac{{{\sigma _p}}}{{\left| {{S_2}} \right|}} + \frac{{{\sigma _{\bf{\Theta }}} + 1}}{{\left| {{S_3}} \right|}}} \right)\left( {r{{\left\| {\left( {{\bf{\Theta }} - {{\bf{\Theta }}_m}} \right){\bf{d}}} \right\|}_2} + \frac{{2{\sigma _{\bf{\Theta }}}\varepsilon }}{{\left| \Omega \right|}}} \right)
 \end{array}
\end{equation}
or equivalently,
\begin{equation}
\frac{2}{{\left| \Omega \right|}}{U^2} \le \kappa \left( {\frac{{{\sigma _p}}}{{\left| {{S_2}} \right|}} + \frac{{{\sigma _{\bf{\Theta }}} + 1}}{{\left| {{S_3}} \right|}}} \right)\left( {rU + \frac{{2{\sigma _{\bf{\Theta }}}\varepsilon }}{{\left| \Omega \right|}}} \right),   
\end{equation}
which can be further simplified by 
\begin{equation}
\begin{array}{l}
U \le \frac{{\left| \Omega  \right|}}{4}\kappa r\left( {\frac{{{\sigma _p}}}{{\left| {{S_2}} \right|}} + \frac{{{\sigma _{\bf{\Theta }}} + 1}}{{\left| {{S_3}} \right|}}} \right)\\
 + \frac{{\left| \Omega  \right|}}{4}\sqrt {{\kappa ^2}{r^2}{{\left( {\frac{{{\sigma _p}}}{{\left| {{S_2}} \right|}} + \frac{{{\sigma _{\bf{\Theta }}} + 1}}{{\left| {{S_3}} \right|}}} \right)}^2} - \frac{{8\kappa {\sigma _{\bf{\Theta }}}\varepsilon }}{{{{\left| \Omega  \right|}^2}}}\left( {\frac{{{\sigma _p}}}{{\left| {{S_2}} \right|}} + \frac{{{\sigma _{\bf{\Theta }}} + 1}}{{\left| {{S_3}} \right|}}} \right)} 
\end{array}.
\end{equation}
Putting the above results into Eq. \eqref{def1} gives
\begin{equation}
\begin{array}{*{20}{l}}
{\left| {{L_p}\left( {{\bf{\Theta }},{\bf{p}};{{{\bf{\hat x}}}_k}} \right) - {L_p}\left( {{{\bf{\Theta }}_m},{\bf{p}};{{{\bf{\hat x}}}_k}} \right)} \right|}\\
{ \le {\sigma _p}{{\left\| {\left( {{\bf{\Theta }} - {{\bf{\Theta }}_m}} \right){\bf{p}}} \right\|}_2}}\\
{ \le {\sigma _p}\left( {r{{\left\| {\left( {{\bf{\Theta }} - {{\bf{\Theta }}_m}} \right){\bf{d}}} \right\|}_2} + \frac{{2{\sigma _{\bf{\Theta }}}\varepsilon }}{{\left| \Omega  \right|}}} \right)}\\
\begin{array}{l}
 \le \frac{{\left| \Omega  \right|{\sigma _p}r}}{4}\left( {\frac{{{\sigma _p}}}{{\left| {{S_2}} \right|}} + \frac{{{\sigma _{\bf{\Theta }}} + 1}}{{\left| {{S_3}} \right|}}} \right) + {\sigma _{\bf{\Theta }}}{\sigma _p}\varepsilon \\
 + \frac{{\left| \Omega  \right|{\sigma _p}}}{4}\sqrt {{r^2}{{\left( {\frac{{{\sigma _p}}}{{\left| {{S_2}} \right|}} + \frac{{{\sigma _{\bf{\Theta }}} + 1}}{{\left| {{S_3}} \right|}}} \right)}^2} + \left( {\frac{{{\sigma _p}}}{{\left| {{S_2}} \right|}} + \frac{{{\sigma _{\bf{\Theta }}} + 1}}{{\left| {{S_3}} \right|}}} \right)\frac{{8{\sigma _{\bf{\Theta }}}\varepsilon }}{{\left| \Omega  \right|}}} 
\end{array}
\end{array}.
\end{equation}
which completes the proof of Theorem \ref{thmss1}.
\end{proof}

Theorem \ref{thmss1} shows that the generalization error of NSFP decreases with the reciprocal of the number of point clouds (${\left| {{S_2}} \right|}$ and ${\left| {{S_3}} \right|}$), demonstrating its superior performance in the large-scale scene flow estimation (please see Tables \ref{tab:mean_time_waymo_3d} and \ref{tab:mean_time_argoverse_3d}), where $\left| {{S_2}} \right| \to \infty$ and $\left| {{S_3}} \right| \to \infty$, demonstrating the effectiveness of NSFP in the large-scale settings. We further provide the analysis for the MNSF method in the following.
 
\begin{theorem}
\label{thmss2}
Let {\small ${{\bf{\Theta }}_{{\bf{fusion}}}} = {\left[ {{\bf{\Theta }}_1^{\top},{\bf{\Theta }}_2^{\top}} \right]^\top}$} denote the parameters of the fusion model. For the proposed multi-frame scheme (MNSF), with high probability, its uniform stability ($\beta_{\rm MNSF}$) is bounded by 
\begin{equation}
\beta_{\rm MNSF} \le \beta_{\rm NSFP} + O\left( {{\textstyle{1 \over {\left| {{S_2}} \right|}}}} \right) , 
\end{equation}
where $\footnotesize O\left( {{\textstyle{1 \over {\left| {{S_2}} \right|}}}} \right)= \frac{{4{\kappa ^2}\sigma _{{S_3}}^2}}{{\lambda \left| {{S_2}} \right|}} + \left( {\frac{{8{\kappa ^2}\sigma _{{S_3}}^2}}{\lambda } + 2{\sigma _{{S_3}}}} \right)\sqrt {\frac{{\ln {1 \mathord{\left/
 {\vphantom {1 \delta }} \right.
 \kern-\nulldelimiterspace} \delta }}}{{2\left| {{S_2}} \right|}}}$ and $ \lambda  = {\textstyle{{\left\| {{{\bf{\Theta }}_2}{{\bf{\Theta }}_b}} \right\|_2^2} \over {\left\| {{{\bf{\Theta }}_1}{{\bf{\Theta }}_f} + {\bf{I}}} \right\|_2^2}}}$. Variables $\kappa$, $\sigma_{S_3}$, and $\delta$ can be considered as constants.
\end{theorem}
\begin{proof}
With the theoretical results, we are ready to prove Theorem \ref{thmss2}. Let {\small ${{\bf{\Theta }}_{{\bf{fusion}}}} = {\left[ {{\bf{\Theta }}_1^{\top},{\bf{\Theta }}_2^{\top}} \right]^\top}$} denote the parameters of the fusion model. Considering a linear fusion function and inverter (defined by Eq. \eqref{eq:optim_reverse}), we have
\begin{equation}
\small
{\bf{\Theta }}\left[ {\begin{array}{*{20}{c}}
{\bf f}\\
{{{\bf f}^{{'}}}}
\end{array}} \right] = \left[ {\begin{array}{*{20}{c}}
{{{\bf{\Theta }}_1}}&{{{\bf{\Theta }}_2}}
\end{array}} \right]\left[ {\begin{array}{*{20}{c}}
{\bf f}\\
{{{\bf f}^{'}}}
\end{array}} \right] = \left[ {\begin{array}{*{20}{c}}
{{{\bf{\Theta }}_1}}&{{{\bf{\Theta }}_2}}
\end{array}} \right]\left[ {\begin{array}{*{20}{c}}
{{{\bf{\Theta }}_f}{\bf{p}}}\\
{ - {{\bf{\Theta }}_b}{\bf{p}}}
\end{array}} \right]
\label{thmss2eq1}
\end{equation}
Then, using Eq. \eqref{thmss2eq1}, we can rewrite the loss function $L_p$ in MNSF optimization as
\begin{equation}
\begin{array}{l}
\frac{1}{{\left| {{S_2}} \right|}}\sum\limits_{j = 1}^{\left| {{S_2}} \right|} {\left\| {\left( {{{\bf{\Theta }}_1}{{\bf{\Theta }}_f} - {{\bf{\Theta }}_2}{{\bf{\Theta }}_b}} \right){{\bf{p}}_j} + {{\bf{p}}_j} - {{{\bf{\hat x}}}_k}} \right\|_2^2} \\
 \le \left\| {{{\bf{\Theta }}_1}{{\bf{\Theta }}_f}{{\bf{p}}_j} + {{\bf{p}}_j} - {{{\bf{\hat x}}}_k}} \right\|_2^2 + \left\| {{{\bf{\Theta }}_2}{{\bf{\Theta }}_b}{{\bf{p}}_j}} \right\|_2^2\\
 = \left\| {g\left( {\bf{p}} \right) - {{{\bf{\hat x}}}_k}} \right\|_2^2 + \lambda \left\| {g\left( {\bf{p}} \right)} \right\|_2^2
\end{array}
\label{thmss2eq2}
\end{equation}
where $\lambda  = {\textstyle{{\left\| {{{\bf{\Theta }}_2}{{\bf{\Theta }}_b}} \right\|_2^2} \over {\left\| {{{\bf{\Theta }}_1}{{\bf{\Theta }}_f} + {\bf{I}}} \right\|_2^2}}}$.
With Eq. \eqref{thmss2eq2} and Theorem 12 \cite{bousquet2002stability}, we finally obtain the theoretical results shown in Theorem \ref{thmss2}.    
\end{proof}

\begin{remark}
As demonstrated in Eq. \eqref{thmss2eq2}, by employing an appropriate fusion strategy, our proposed MNSF emerges as a polynomial function of the approach utilized in NSFP, revealing a straightforward but essential variation of the NSFP algorithm.
\end{remark}

Theorem \ref{thmss2} reveals two key aspects of MNSF based on loss function in Eq. \eqref{ea22}: 1) The algorithm's generalization error is inversely proportional to the number of point clouds, indicating its efficacy with large-scale point clouds (please see Tables \ref{tab:mean_time_waymo_3d} and \ref{tab:mean_time_argoverse_3d}); 2) Theoretical analysis shows that MNSF's generalization error upper bound is on par with NSFP's when $\left| {{S_2}} \right| \to \infty$. This indicates that adding the $t\text{-}1$ frame into the optimization maintains, and even enhances, generalization, as supported by the case study in Section \ref{case_study}.

\section{Experiments}
\label{sc:exp}

In this section, we evaluate the proposed method on large-scale and realistic autonomous driving scenes. Specifically, we first introduce datasets and evaluation metrics. Then, we compare the proposed method with NSFP, FNSF, and different learning-based methods. Finally, we verify the effectiveness of each component in the proposed method with an ablation study.

~\\
\textbf{Datasets.} We focus on large-scale and lidar-based autonomous driving scenes. To this end, we conduct experiments on the Waymo Open~\cite{sun2020scalability} and the Argoverse~\cite{chang2019argoverse} datasets. Specifically, we follow previous studies \cite{li2021neural, li2023fast} to pre-process these two open-world datasets and generate the pseudo ground truth scene flow.

~\\
\textbf{Metrics.}
We evaluate the performance of the scene flow estimation based on widely used metrics from~\cite{liu2019flownet3d, mittal2020just, wu2020pointpwc, pontes2020scene, li2021neural, li2023fast}.
These metrics are introduced as follows. \\
\textbf{(1) 3D end-point error $\mathcal{E}(m)$} measures the mean absolute distance between the estimated scene flow and the pseudo ground truth scene flow. \\
\textbf{(2) Strict accuracy $Acc_5(\%)$} represents the ratio of points that the absolute point error $\mathcal{E} < 0.05$m or the relative point error $\mathcal{E}^{\prime} < 0.05$. \\
\textbf{(3) Relaxed accuracy $Acc_{10}(\%)$} represents the ratio of points that the absolute point error $\mathcal{E} < 0.1$m or the relative point error $\mathcal{E}^{\prime} < 0.1$. \\
\textbf{(4) Outlier $Outliers(\%)$} represents the ratio of points that the absolute point error $\mathcal{E} > 0.3$m or the relative point error $\mathcal{E}^{\prime} > 0.1$. In this way, $Inliers= 1 - Outliers$.\\
\textbf{(5) Angle error $\theta_\epsilon(rad)$} measures the mean angle error between the estimated scene flow and the pseudo ground truth scene flow.\\
\textbf{(6) Inference time $t(ms)$} measures the computation time for the scene flow estimation.

\begin{table*}[t]
\caption[]{\textbf{Evaluation on the Waymo Open Scene Flow dataset.} 
We follow previous studies \cite{li2021neural, li2023fast} to pre-process the Waymo Open dataset and generate 202 testing examples. Each point cloud contains 8k-144k points. The upper tabular between {\color{sol_blue}\textbf{blue bars}} are evaluated with the full point cloud as the input, and the lower tabular between {\color{beer_orange}\textbf{orange bars}} are evaluated with random samples 8,192 points as the input. The best performance has been bold, and the second-best performance has been underlined. $\uparrow$ indicates larger values are better while $\downarrow$ indicates smaller values are better. 
}
    \centering
    \begin{adjustbox}{width=\linewidth}
    \begin{tabular}{@{}clcccccccccccc@{}}
        \toprule
        &\multirow{1}{*}{\thead{\normalsize Method}}
        &\multirow{1}{*}{\thead{\normalsize Supervision}} 
        &\multirow{1}{*}{\thead{\normalsize Train set size}}
        &\multirow{1}{*}{\thead{\normalsize ${\mathcal{E}}$${(m)}\downarrow$}} 
        &\multirow{1}{*}{\thead{\normalsize ${Acc_5}$${(\%)}\uparrow$}} 
        &\multirow{1}{*}{\thead{\normalsize ${Acc_{10}}$${(\%)}\uparrow$}} 
        &\multirow{1}{*}{\thead{\normalsize $Outliers(\%)\downarrow$}} 
        &\multirow{1}{*}{\thead{\normalsize ${\theta_{\epsilon}}$${(rad)}\downarrow$}} 
        &\multicolumn{1}{c}{\thead{\normalsize $t~  (ms)\downarrow$}} \\
        \arrayrulecolor{sol_blue}\toprule[0.5ex]
        & NSFP~\cite{li2021neural} & \textit{Self} & 0 & 0.087 & 78.21
        & 90.18 
        & 37.44 
        & 0.295 
        & 15310  \\
        & NSFP (linear) & \textit{Self} & 0 & 0.153 & 60.28 
        & 75.89
        & 53.19
        & 0.353
        & 7964 \\
        & FNSF & \textit{Self} & 0 & 0.075 & \underline{85.34} 
        & \underline{92.54}
        & \underline{32.80}
        & \underline{0.286} 
        & \underline{609} \\
        & FNSF (linear) & \textit{Self} & 0 & 0.114 & 71.03 
        & 85.54 
        & 43.59
        & 0.339 
        & \textbf{451} \\
        & FNSF (joint) & \textit{Self} & 0 & 0.081 & 82.61 
        & 92.16 
        & 34.58
        & 0.291 
        & 920 \\
        & FNSF (temporal encoding) & \textit{Self} & 0 & 0.079 & 82.75 
        &  92.22
        & 0.339
        &  0.291
        & 1011 \\
        & Ours (cycle consistency) & \textit{Self} & 0 & \underline{0.071} & 81.09 
        &  91.58
        & 35.28
        &  0.300
        &  1831 \\
        & Ours & \textit{Self} & 0 & \textbf{0.066} & \textbf{87.16} 
        & \textbf{93.39} 
        & \textbf{30.89} 
        & \textbf{0.273} 
        & 989\\
        \arrayrulecolor{sol_blue}\toprule[0.5ex]
        \arrayrulecolor{beer_orange}\toprule[0.5ex]
        & FLOT~\cite{puy20flot} & \textit{Full} & 18,000 & 0.694 & 2.62 & 11.89 
        & 94.74 & 0.792 
        & 133 \\
        & 3DFlow~\cite{wang2022matters} & \textit{Full} & 18,000 & 2.088 & 1.60 & 4.92 
        & 98.94 & 1.845
        & \textbf{80} \\
        & GMSF~\cite{zhang2023gmsf} & \textit{Full} & 18,000 & 8.058 & 0.00 & 0.01
        & 99.96 
        & 1.341
        & 245  \\
        & SCOOP~\cite{lang2023scoop} & \textit{Self} & 1,800 & 0.313 & 41.86 & 65.02
        & 64.71 
        & 0.474
        & 558  \\
        & NSFP~\cite{li2021neural} (8,192 pts) & \textit{Self} & 0 & \underline{0.109} & 64.63 & 81.82 
        & 45.60 
        & 0.338
        & 4450 \\
        & FNSF~\cite{li2023fast} (8,192 pts) & \textit{Self} & 0 & 0.110 & \underline{72.78} & \underline{87.73} 
        & \underline{39.75} 
        & \underline{0.324}
        & \underline{84} \\
        & Ours (8,192 pts) & \textit{Self} & 0& \textbf{0.102} & \textbf{79.42} & \textbf{90.87} 
        & \textbf{36.51}
        & \textbf{0.321}
        & 160  \\
        \arrayrulecolor{beer_orange}\toprule[0.5ex]
        \arrayrulecolor{black}\bottomrule
    \end{tabular}
    \end{adjustbox}
    \label{tab:mean_time_waymo_3d}
\end{table*}

\begin{table*}[t]
\caption[]{\textbf{Evaluation on the Argoverse Scene Flow dataset.} 
We follow previous studies \cite{li2021neural, li2023fast} to pre-process the Argoverse dataset and generate 508 testing examples. Each point cloud contains 30k-70k points. The upper tabular between {\color{sol_blue}\textbf{blue bars}} are evaluated with the full point cloud as the input, and the lower tabular between {\color{beer_orange}\textbf{orange bars}} are evaluated with random samples 8,192 points as the input. The best performance has been bold, and the second-best performance has been underlined. $\uparrow$ indicates larger values are better while $\downarrow$ indicates smaller values are better. 
}
    \centering
    \begin{adjustbox}{width=\linewidth}
    \begin{tabular}{@{}clcccccccccccc@{}}
        \toprule
        &\multirow{1}{*}{\thead{\normalsize Method}}
        &\multirow{1}{*}{\thead{\normalsize Supervision}} 
        &\multirow{1}{*}{\thead{\normalsize Train set size}}
        &\multirow{1}{*}{\thead{\normalsize ${\mathcal{E}}$${(m)}\downarrow$}} 
        &\multirow{1}{*}{\thead{\normalsize ${Acc_5}$${(\%)}\uparrow$}} 
        &\multirow{1}{*}{\thead{\normalsize ${Acc_{10}}$${(\%)}\uparrow$}} 
        &\multirow{1}{*}{\thead{\normalsize $Outliers(\%)\downarrow$}} 
        &\multirow{1}{*}{\thead{\normalsize ${\theta_{\epsilon}}$${(rad)}\downarrow$}} 
        &\multicolumn{1}{c}{\thead{\normalsize $t~  (ms)\downarrow$}} \\
        \arrayrulecolor{sol_blue}\toprule[0.5ex]
        & NSFP~\cite{li2021neural} & \textit{Self} & 0 & 0.083 & 75.15
        & 86.49 
        & 39.13 
        & 0.361 
        & 15214  \\
        & NSFP (linear) & \textit{Self} & 0 & 0.107 & 58.39 
        & 76.39
        & 55.21
        & 0.337 
        & 2994 \\
        & FNSF & \textit{Self} & 0 & \underline{0.049} & \underline{87.04} 
        & \underline{94.08}
        & \underline{29.88}
        & \underline{0.307} 
        & \underline{472} \\
        & FNSF (linear) & \textit{Self} & 0 & 0.082 & 71.03 
        & 87.32
        & 41.64
        & 0.338 
        & \textbf{396} \\
        & FNSF (joint) & \textit{Self} & 0 & 0.050 & 84.77 
        & 93.46
        & 31.77
        & 0.319 
        & 793 \\
        & FNSF (temporal encoding) & \textit{Self} & 0 &  0.052 &  85.14
        &  93.26
        & 31.93
        &  0.322
        & 879 \\
        & Ours (cycle consistency) & \textit{Self} & 0 & 0.054 &  83.26
        &  92.36
        & 32.81
        &  0.325
        & 1432 \\
        & Ours & \textit{Self} & 0 & \textbf{0.044} & \textbf{88.75} 
        & \textbf{94.83} 
        & \textbf{28.86} 
        & \textbf{0.299} 
        & 851\\
        \arrayrulecolor{sol_blue}\toprule[0.5ex]
        \arrayrulecolor{beer_orange}\toprule[0.5ex]
        & FLOT~\cite{puy20flot} & \textit{Full} & 18,000 & 0.767 & 2.33 & 9.91 
        & 96.19 & 0.971 
        & 130 \\
        & 3DFlow~\cite{wang2022matters} & \textit{Full} & 18,000 & 1.672 & 3.08 & 9.22 
        & 96.92 & 1.845
        & \textbf{82} \\
        & GMSF~\cite{zhang2023gmsf} & \textit{Full} & 18,000 & 9.089 & 0.00 & 0.01
        & 99.99 
        & 1.781
        & 247  \\
        & SCOOP~\cite{lang2023scoop} & \textit{Self} & 1,800 & 0.248 & 39.09 & 62.56
        & 68.81 
        & 0.481
        & 542  \\
        & NSFP~\cite{li2021neural} (8,192 pts) & \textit{Self} & 0 & \underline{0.077} & 63.39 & 81.26
        & 46.72 
        & \underline{0.366}
        & 4390 \\
        & FNSF~\cite{li2023fast} (8,192 pts) & \textit{Self} & 0 & 0.081 & \underline{75.87} & \underline{87.85}
        & \underline{39.10} 
        & 0.372
        & \underline{83} \\
        & Ours (8,192 pts) & \textit{Self} & 0& \textbf{0.069} & \textbf{82.10} & \textbf{92.93} 
        & \textbf{32.86}
        & \textbf{0.344}
        & 157 \\
        \arrayrulecolor{beer_orange}\toprule[0.5ex]
        \arrayrulecolor{black}\bottomrule
    \end{tabular}
    \end{adjustbox}
    \label{tab:mean_time_argoverse_3d}
\end{table*}

~\\
\textbf{Implementation details.}
We introduce details of implementation for each compared method.\\
\textbf{(1) NSFP~\cite{li2021neural}.} 
We follow NSFP~\cite{li2021neural} to use an 8-layer MLP to estimate the scene flow. Specifically, the weights of the MLP are randomly initialized before optimizing each pair of point clouds. \\
\textbf{(2) NSFP (linear).} 
Following \cite{li2023fast}, we implement NSFP via a linear model with complex positional encodings, namely (\textbf{NSFP (linear)}). Specifically, using 8 linear layers and computing the Kronecker product of the per-axis encoding.\\
\textbf{(3) FNSF \cite{li2023fast}.} For a fair comparison, we implement FNSF with an 8-layer MLP. The grid cell size of FNSF is set to 0.1 meters. \\
\textbf{(4) FNSF (linear).} We also implement FNSF via a linear model with complex positional encodings, namely (\textbf{FNSF (linear)}). The settings of the linear model and positional encodings are the same as in \textbf{NSFP (linear).} \\
\textbf{(5) FNSF (joint).} To demonstrate the necessity of a dedicated strategy for utilizing temporal information, we use a single FNSF to jointly estimate the previous flow ($t\text{-}1 \,{\rightarrow}\, t$) and the current flow ($t \,{\rightarrow}\, t\text{+}1$), namely (\textbf{FNSF (joint)}).\\
\textbf{(6) FNSF (temporal encoding).} Following \cite{zheng2023neuralpci}, we also use an FNSF to estimate the previous flow ($t\text{-}1 \,{\rightarrow}\, t$) and the current flow ($t \,{\rightarrow}\, t\text{+}1$) with temporal encoding. Specifically, such a model encodes both the spatial and temporal coordinates of multi-frame point clouds, namely (\textbf{FNSF (temporal encoding)}).\\
\textbf{(7) Ours.} We implement models $g_f$ and $g_b$ with 8-layer MLPs. These two models are independently trained. We simplify the model $g_{\rm invert}$ as a constant model and adopt a 3-layer MLP as the fusion model $g_{\rm fusion}$. The architecture of the fusion model is discussed in Section~\ref{sc:exp:ablation}. The grid cell size of FNSF is consistently set to 0.1 meters.\\
\textbf{(8) Ours (cycle consistency).} We also implement the proposed method with a cycle consistency constraint in \cite{li2021neural}, which aims to improve the smoothness of the scene flow estimation. To this end, the cycle consistency constraint controls the distances between the estimated point cloud and the original point cloud.\\
\textbf{(9) \textbf{FLOT~\cite{puy20flot}}, \textbf{3DFlow \cite{wang2022matters}}, and GMSF~\cite{zhang2023gmsf}} are supervised learning-based methods trained on the synthetic FlyingThings3D~\cite{mayer2016large} and the KITTI~\cite{menze2015object} datasets. On the other hand, \textbf{SCOOP \cite{lang2023scoop}} is a self-supervised method. These models are directly evaluated with pre-trained models and official codes released by the authors.  

All experiments are conducted on a computer with a single NVIDIA RTX 3090Ti GPU and a Gen Intel (R) 24-Core (TM) i9-12900K CPU. 
We implement all compared models based on PyTorch.

\begin{figure*}[t]
\centering
\begin{tabular}{c c}
\multicolumn{2}{c}{\includegraphics[width=0.99\linewidth]{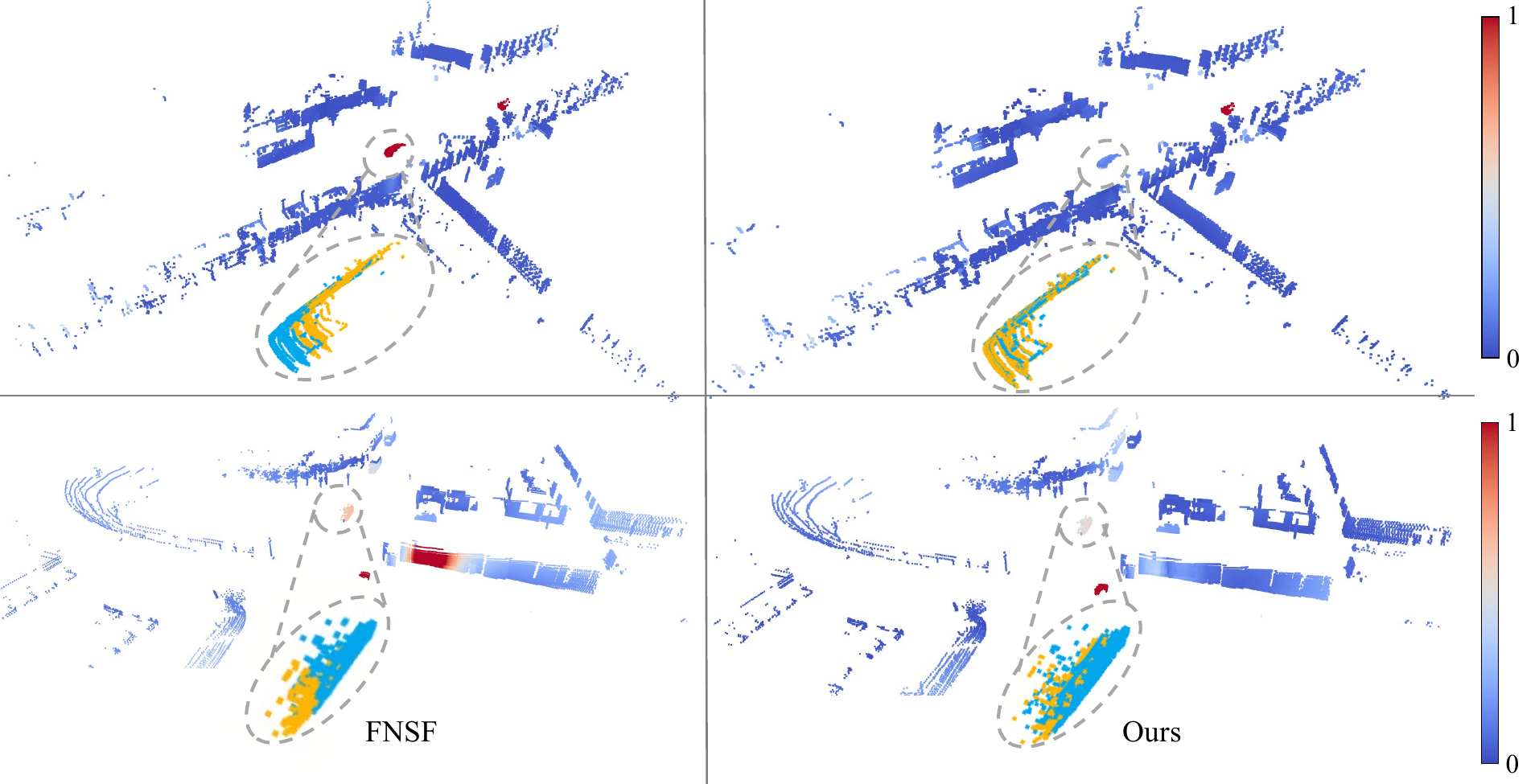}} \\ 
\end{tabular}
\caption{Visual comparison between FNSF and the proposed method on the Argoverse dataset. For each point, color represents the normalized 3D end-point error $\mathcal{E}$. In this way, blue indicates the estimation of the flow is accurate. The detailed view demonstrates two point clouds aligned by the estimated flow.}
\label{fig:fig3}
\end{figure*}

\subsection{Comparison of Performance}
\label{sc:exp:benchmark}
We evaluate and compare the proposed method with various state-of-the-art methods on the Waymo Open (Table \ref{tab:mean_time_waymo_3d}) and the Argoverse (Table \ref{tab:mean_time_argoverse_3d}) datasets. For simplicity, we represent results on the Waymo Open (xx) and the Argoverse (yy) as xx/yy in the following paragraph. Figure \ref{fig:fig3} shows the visual comparison between FNSF and the proposed method on the Argoverse dataset.

~\\
\textbf{Dense scene flow estimation.}
The ability to estimate dense scene flow is crucial, because each LiDAR scan often contains 100K - 1000K points in real-world autonomous driving scenarios \cite{jund2021scalable}. Therefore, we evaluate scene flow methods with the full point cloud as the input. NSFP achieves 78.21/75.15\% strict accuracy, but the computation time costs 15310/15214 ms. To accelerate the optimization process, NSFP (linear) replaces the MLP with a linear model and positional encoding. In this way, NSFP (linear) speedups the optimization process almost two times and achieves worse performance compared to NSFP, \emph{i.e.}, accuracy decreases by about 15\%. FNSF achieves almost 30$\times$ speedup and improves the strict accuracy to 85.34/87.04\%. Meanwhile, FNSF (linear) slightly accelerates FNSF, suffering from a relatively large drop in performance. 

All the above methods only use two frames ($t$ and $t\text{+}1$) and neglect to utilize previous frames. To this end, FNSF (joint) estimates the previous flow ($t\text{-}1 \,{\rightarrow}\, t$) and the current flow ($t \,{\rightarrow}\, t\text{+}1$) at the same time. However, such an intuitive scheme obtains worse strict accuracy (82.61/84.77\%) than FNSF. The interpretation of this phenomenon is that a single MLP fails to encode different motion fields simultaneously. In other words, points in the frame $t\text{-}1$ and the frame $t$ may have the same position $(x, y, z)$ with different motion fields. Thus, it is difficult for the MLP to learn from these inconsistent samples \cite{liu2023towards}. In contrast, the proposed method exploits valuable temporal information from previous frames and outperforms FNSF and FNSF (joint). In this way, the proposed method achieves a balance between performance and inference time.

~\\
\textbf{OOD generalizability.}
In order to conduct a fair comparison with learning-based methods~\cite{puy20flot, wang2022matters, zhang2023gmsf, lang2023scoop}, we further extend the proposed method to process a reduced number of points, \emph{i.e.}, 8,192 points. Current learning-based methods could only process a fixed and small number of points due to their cumbersome networks \cite{peng2023delflow, zhang2023gmsf}, \emph{e.g.}, transformer-based architectures. To this end, these methods have to downsample or divide the entire lidar scan into smaller subsets/regions. Then, these learning-based methods can be iteratively used to predict the scene flow of each subset point cloud. In this way, such a compromising point cloud pre-process operation limits the generalization ability of learning-based methods on the large-scale OOD data and may lead to out-of-memory issues \cite{jund2021scalable, chodosh2023re}. 

Table~\ref{tab:mean_time_waymo_3d} and Table~\ref{tab:mean_time_argoverse_3d} show that supervised learning-based methods, including FLOT, 3DFlow, and GMSF, achieve limited performance on large-scale autonomous driving Waymo Open and Argoverse datasets. It is because of the huge domain shift between the training data and testing data~\cite{pontes2020scene, li2021neural, najibi2022motion, dong2022exploiting, jin2022deformation, chodosh2023re}. In contrast, the self-supervised SCOOP outperforms its supervised counterparts and achieves 41.86/39.09\% strict accuracy. However, the performance of SCOOP is still inferior to NSFP and FNSF. The proposed method outperforms FNSF by exploiting and utilizing temporal information from multi-frames. Although the computation cost of 3DFlow is the lowest among all compared methods, the proposed method achieves a balance between the performance and computational complexity. These experimental results and analysis indicate that the proposed method is robust for OOD data and is applicable to real-world autonomous driving scenarios.

\setlength{\tabcolsep}{3pt} 
\begin{table*}[t]
	\caption{\textbf{Performance of the proposed method with different components on the Waymo Open dataset.} All compared methods are evaluated with the full point cloud as the input. $\uparrow$ indicates larger values are better while $\downarrow$ indicates smaller values are better.}
		\centering
		\begin{tabular}[b]{c|c c c|cccccc}
			\toprule
			Model & Multi-frame & $g_{\rm invert}$ & $g_{\rm fusion}$ & ${\mathcal{E}}$${(m)}\downarrow$ & ${Acc_{5}}$${(\%)}\uparrow$ &${Acc_10}$${(\%)}\uparrow$
        &$Outliers(\%)\downarrow$ 
        &${\theta_{\epsilon}}$${(rad)}\downarrow$ 
        &$t~  (ms)\downarrow$ 
			\\
			\midrule
      		FNSF &  & &  & 0.075  & 85.34  & 92.54
        & 32.80
        & 0.286
        & \textbf{609}
			\\
   			(a) &  &  & $\checkmark$ & 0.083 & 84.06 &92.58 & 33.52 & 0.325& 734
			\\
			(b) & $\checkmark$ & $\checkmark$ &  & 0.070 &82.94 & 92.64 & 32.89 & 0.284 & 613
			\\
			(c) & $\checkmark$ &  & $\checkmark$ & 0.088 & 78.96 & 88.97 & 37.43 & 0.320 & 987 
			\\
			(d) & $\checkmark$ & $\checkmark$ & $\checkmark$ & \textbf{0.066} & \textbf{87.16} & \textbf{93.39} & \textbf{30.89} & \textbf{0.273} & 989
            \\
			\bottomrule
		\end{tabular}
	\label{tab:ablation}
\end{table*}

~\\
\textbf{Discussions about learning-based methods.}
Learning-based scene flow methods \cite{puy20flot, liu2019flownet3d, liu2019meteornet, wang2022matters, zhang2023gmsf, peng2023delflow} have exhibited remarkable speed and performance on small-scale synthetic datasets, \emph{e.g.}, KITTI Scene Flow\footnote{Point clouds in the KITTI dataset are limited to a specific range (35-meter within the scene center) with a small number of points (2048 or 8192 points).}~\cite{menze2015object} and FlyingThings3D~\cite{mayer2016large} datasets. However, these methods heavily rely on the high consistency between training scenarios and testing scenarios \cite{pontes2020scene, li2021neural, najibi2022motion, dong2022exploiting, jin2022deformation, chodosh2023re}, \emph{e.g.}, viewpoints and coordinate systems. Thus, it is a challenge to use these learning-based methods in real-world applications, where training scenarios and testing scenarios are often inconsistent. To this end, we propose a multi-frame scheme based on FNSF, instead of learning based-methods.

~\\
\textbf{Cycle consistency constraint.} We conduct experiments to figure out whether the proposed method can be further improved by the cycle/temporal consistency loss, because it is common practice to encourage the trajectory of point cloud to be smooth \cite{liu2019flownet3d, mittal2020just, wang2022neural} for multi-frame point clouds, by constraining the distance between point clouds from different frames. To this end, a temporal consistency loss or a cycle consistency loss is usually used during the training process of point cloud models. Table \ref{tab:mean_time_waymo_3d} and Table \ref{tab:mean_time_argoverse_3d} show that adding the cycle consistency loss decreases the performance of the proposed method, \emph{i.e.}, strict accuracy decreasing from 87.16/88.75\% to 81.09/83.26\%. In addition, the cycle consistency loss significantly increases the computational complexity, and the inference time costs 1831 ms. Thus, the cycle/temporal consistency loss is not necessary in our case. Such a finding also verifies the empirical observation in \cite{li2023fast}. Therefore, we implicitly enforce cycle/temporal smoothness, instead of explicitly constraining cycle/temporal smoothness.

~\\
\textbf{Temporal encoding.} We also compare the proposed multi-frame scheme with the temporal encoding strategy, because temporal encoding is useful to process point cloud sequences \cite{wang2022neural, zheng2023neuralpci}. As aforementioned, it is difficult for FNSF (joint) to distinguish point clouds from different frames. To mitigate this issue, we use temporal encoding and concatenate the temporal coordinate into the spatial coordinate, \emph{i.e.}, obtaining a 4D point cloud. In this way, we construct FNSF (temporal encoding) to jointly estimate the previous flow ($t\text{-}1 \,{\rightarrow}\, t$) and the current flow ($t \,{\rightarrow}\, t\text{+}1$). Table \ref{tab:mean_time_waymo_3d}) and Table \ref{tab:mean_time_argoverse_3d} show that FNSF (temporal encoding) slightly outperforms FNSF (joint). Such experimental result indicates that using temporal encoding partially addresses the issue in FNSF (joint) with limited performance improvement. However, FNSF (joint) is still \textit{inferior to} the proposed method. The interpretation is that temporal encoding may be more suitable for long sequence point clouds than short sequence point clouds \cite{wang2022neural}. Therefore, the proposed method provides a promising solution to multi-frame point cloud scene flow estimation.

\setlength{\tabcolsep}{1pt} 
\begin{table}[t]
\caption[]{\textbf{Performance of different architectures of the temporal fusion model on the Waymo Open dataset.} All compared methods are evaluated with the full point cloud as the input. $\uparrow$ indicates larger values are better while $\downarrow$ indicates smaller values are better. 
}
    \centering
    \begin{tabular}{@{}clcccccccc@{}}
        \toprule
        &\multirow{1}{*}{\thead{\normalsize Operation}}
        &\multirow{1}{*}{\thead{\normalsize ${\mathcal{E}}$${(m)}\downarrow$}} 
        &\multirow{1}{*}{\thead{\normalsize ${Acc_5}$${(\%)}\uparrow$}} 
        &\multirow{1}{*}{\thead{\normalsize ${Acc_{10}}$${(\%)}\uparrow$}} 
        &\multirow{1}{*}{\thead{\normalsize ${\theta_{\epsilon}}$${(rad)}\downarrow$}}  \\
        \midrule
        & Mean & 0.070 & 82.55 & 92.64
        & 0.285  \\
        & Weighted sum & 0.097 & 84.18 & 92.42 
        & 0.286\\
        & MLP & \textbf{0.066} & \textbf{87.16} & \textbf{93.39}
        & \textbf{0.273} \\
        \arrayrulecolor{black}\bottomrule
    \end{tabular}
    \label{tab:fusion}
\end{table}

\begin{table}[t]
\caption[]{\textbf{Performance of different depths of the temporal fusion model on the Waymo Open dataset.} All compared methods are evaluated with the full point cloud as the input. $\uparrow$ indicates larger values are better, while $\downarrow$ indicates smaller values are better. 
}
    \centering
    \begin{adjustbox}{width=0.9\linewidth}
    \begin{tabular}{@{}clcccccccc@{}}
        \toprule
        &\multirow{1}{*}{\thead{\normalsize Setting}}
        &\multirow{1}{*}{\thead{\normalsize ${\mathcal{E}}$${(m)}\downarrow$}} 
        &\multirow{1}{*}{\thead{\normalsize ${Acc_5}$${(\%)}\uparrow$}} 
        &\multirow{1}{*}{\thead{\normalsize ${Acc_{10}}$${(\%)}\uparrow$}} 
        &\multirow{1}{*}{\thead{\normalsize ${\theta_{\epsilon}}$${(rad)}\downarrow$}}  \\
        \midrule
        & 2 layers & 0.069 & 86.84  & 93.07
        &  0.286 \\
        & 3 layers & \textbf{0.066} & \textbf{87.16} & \textbf{93.39}
        & \textbf{0.273}  \\
        & 5 layers & 0.068 & 86.33 & 93.16 
        & 0.281\\
        & 7 layers & 0.107 & 83.50 & 92.30 
        & 0.303\\
        \arrayrulecolor{black}\bottomrule
    \end{tabular}
    \end{adjustbox}
    \label{tab:fusion_layer}
\end{table}

\begin{table}[t]
\caption[]{\textbf{Performance of different frame numbers on the Waymo Open dataset.} All compared methods are evaluated with the full point cloud as the input. $\uparrow$ indicates larger values are better, while $\downarrow$ indicates smaller values are better. 
}
    \centering
    \begin{adjustbox}{width=0.9\linewidth}
    \begin{tabular}{@{}clcccccccc@{}}
        \toprule
        &\multirow{1}{*}{\thead{\normalsize Setting}}
        &\multirow{1}{*}{\thead{\normalsize ${\mathcal{E}}$${(m)}\downarrow$}} 
        &\multirow{1}{*}{\thead{\normalsize ${Acc_5}$${(\%)}\uparrow$}} 
        &\multirow{1}{*}{\thead{\normalsize ${Acc_{10}}$${(\%)}\uparrow$}} 
        &\multirow{1}{*}{\thead{\normalsize ${\theta_{\epsilon}}$${(rad)}\downarrow$}}  \\
        \midrule
        & 2 frames & 0.083 & 84.46 & 92.58
        &  0.313 \\
        & 3 frames & \textbf{0.066} & 87.16 & \textbf{93.39}
        & \textbf{0.273}  \\
        & 4 frames & 0.070 & \textbf{87.64} & 93.38 
        & 0.279\\
        & 5 frames & 0.085 & 87.48 & 93.31 
        & 0.291\\
        \arrayrulecolor{black}\bottomrule
    \end{tabular}
    \end{adjustbox}
    \label{tab:frame_numbers}
\end{table}

\begin{figure*}[t]
\centering
\includegraphics[width=0.99\linewidth]{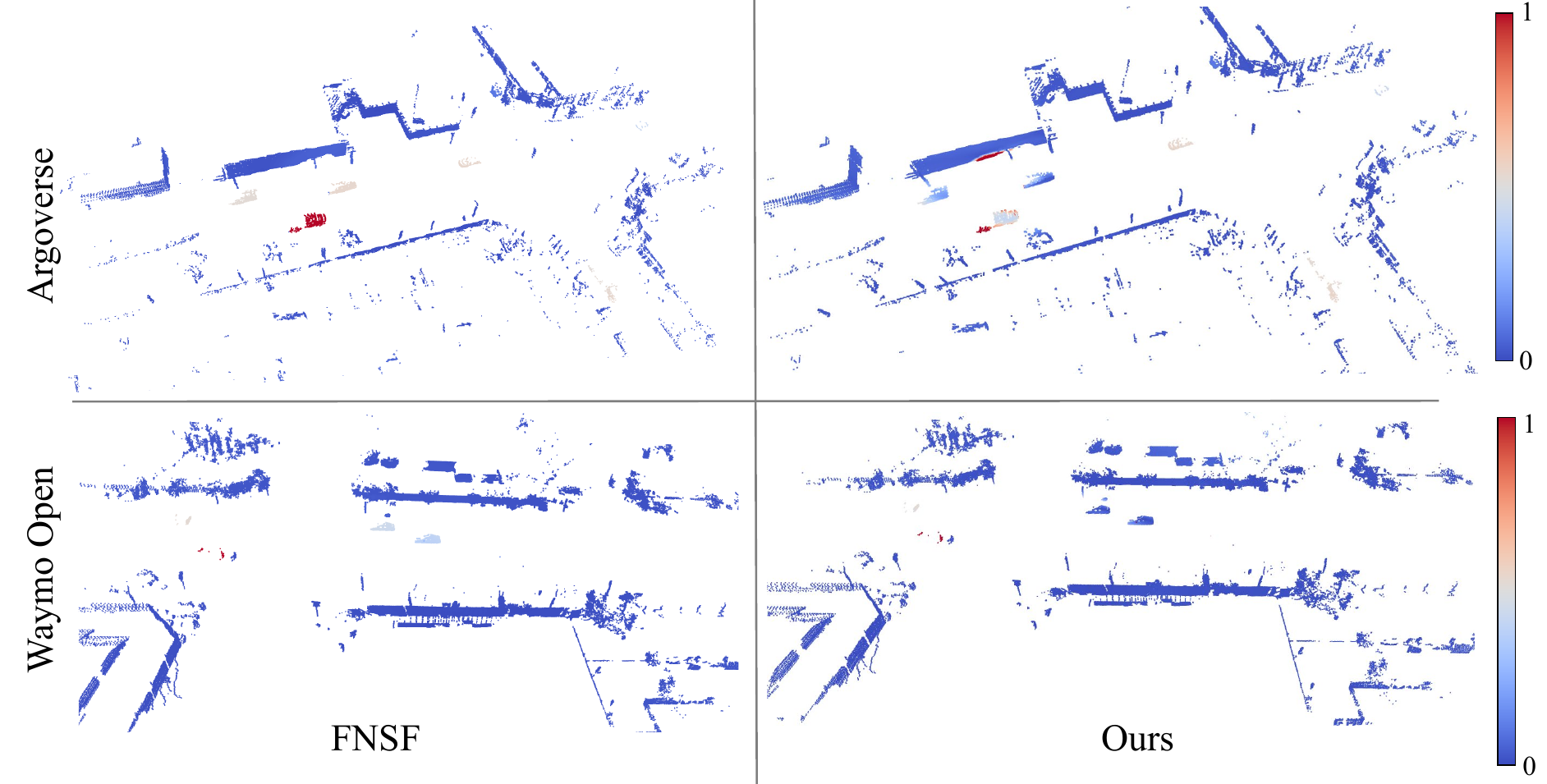} \\ 
\caption{Fast motion cases on the Argoverse and the Waymo Open datasets. Color represents the normalized 3D end-point error $\mathcal{E}$ for each point. In other words, blue indicates the estimation of the flow is accurate.}
\label{fig:fig5}
\end{figure*}

\subsection{Ablation Study}
\label{sc:exp:ablation}

In this section, we first conduct comprehensive experiments to verify the effectiveness of each component in the proposed method on the Waymo Open dataset. Specifically, given the forward and backward flows, the following four models are evaluated: (a) use the model $g_{\rm fusion}$ to refine the forward flow; (b) use the model $g_{\rm invert}$ to invert the backward flow, then directly compute the average of the inverted flow and the forward flow as the fused flow; (c) use the model $g_{\rm fusion}$ to directly fuse the forward and backward flows; (d) equip all components, \emph{i.e.}, the proposed method.

Table \ref{tab:ablation} shows that each component is effective. Model (a) achieves comparable performance with FNSF without exploiting valuable information from previous frames. By coarsely using previous frames, model (b) slightly outperforms FNSF. Although model (c) uses information from previous frames, it performs worse than FNSF. This is because the forward and backward flows represent opposite directions and conflict with each other. Therefore, the direct fusion leads to performance degradation. Combining an inverter model $g_{\rm invert}$ and a fusion model $g_{\rm fusion}$ (\emph{i.e.}, model (d)), achieves better performance than FNSF.

~\\
\textbf{Architecture of the temporal fusion model.} We provide results of the proposed method with different architectures of the temporal fusion model. The temporal fusion model is an average operation, a learnable matrix $W$, and an MLP, respectively. Specifically, mean denotes directly computing the average of the forward and the inverted backward scene flow, \emph{i.e.}, ($\mathbf{f} + \mathbf{f^{'}})/2$. The weighted sum represents using the learnable matrix $W$ to adjust the weights between the forward and the inverted backward scene flow, \emph{i.e.}, $W\mathbf{f} + (I-W)\mathbf{f^{'}}$. In comparison, these two flows are the input to the MLP, and the output is the fused flow. Table \ref{tab:fusion} shows that setting the temporal fusion model as an MLP achieves optimal performance.

~\\
\textbf{Depth of the temporal fusion model.} We illustrate the results of the proposed method with different depths of the temporal fusion model $g_{\rm fusion} \left(\cdot \, \mathbf{\Theta_{\rm fusion}} \right)$. Specifically, the temporal fusion model is set as 2-layer MLP, 3-layer MLP, 5-layer MLP, and 7-layer MLP. Table \ref{tab:fusion_layer} shows that a 3-layer MLP temporal fusion model achieves the optimal performance. Therefore, a relative small layer number of the temporal fusion model could better accomplish the fusion procedure.

~\\
\textbf{Number of frames.} We demonstrate the results of the proposed method with different frame numbers. Specifically, we have point clouds from $t\text{-}(m\text{-}2)$, $\cdots$, $t\text{-}1$, $t$, and $t\text{+}1$ for the $m$-frame setting. We independently train $m\text{-}1$ models, predicting the forward flow $t \,{\,{\rightarrow}\,}\, t\text{+}1$ and $m\text{-}2$ backward flow $t \,{\rightarrow}\, t\text{-}1$, $t \,{\rightarrow}\, t\text{-}2$, $\cdots$, $t \,{\rightarrow}\, t\text{-}(m\text{-}2)$, respectively. Finally, we use a fusion model to 
estimate the final flow, \emph{i.e.}, $t \,{\rightarrow}\, t\text{+}1$. Table \ref{tab:frame_numbers} shows that the multi-frame setting outperforms the 2-frame setting. It verifies that exploiting temporal information from previous frames is useful for scene flow estimation. Table \ref{tab:frame_numbers} also reveals that the contribution of the temporal information is incremental, when the number of frames is larger than three. Such a finding is consistent with the previous work in the optical flow estimation \cite{ren2019fusion}.

\begin{figure}[t]
\centering
\includegraphics[width=0.99\linewidth]{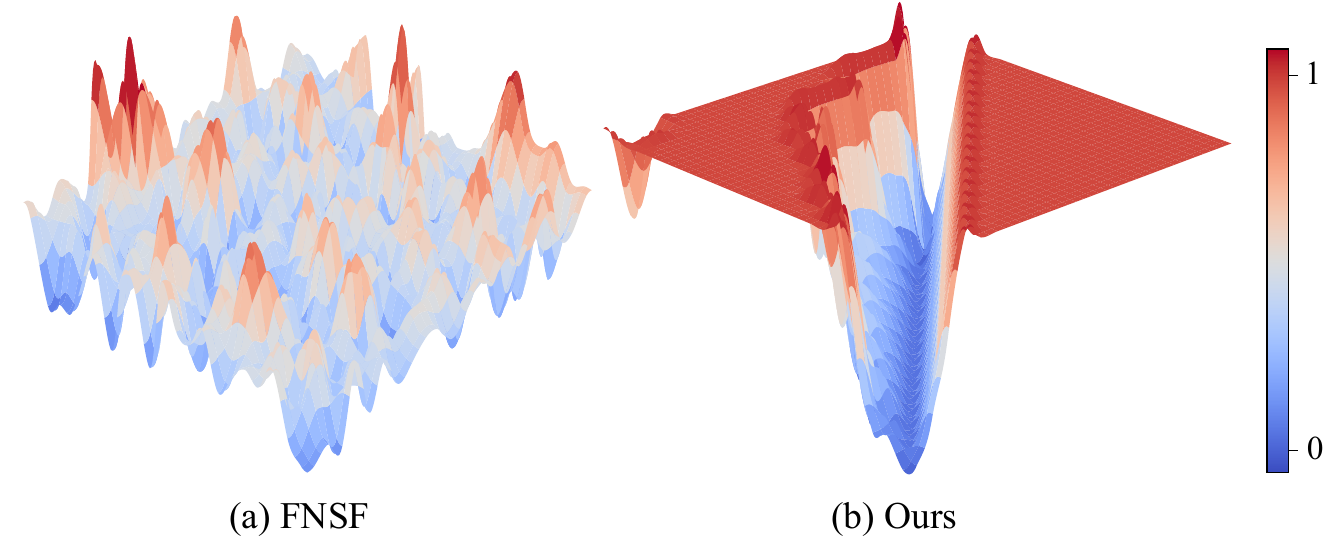} \\ 
\caption{The loss landscapes of FNSF and the proposed method on the Argoverse dataset. Color represents the testing loss. The proposed method eases the scene flow optimization process and has a more flat minimum.}
\label{fig:fig4}
\end{figure}

\subsection{Case study}
\label{case_study}

~\\
\textbf{Fast motion cases.}
The ability to estimate dense scene flow of fast motion is important in real-world autonomous driving. Therefore, we demonstrate the error of the scene flow estimation in fast motion cases. Specifically, we select two fast motion cases from Argoverse and Waymo Open datasets based on the pseudo ground truth scene flow, respectively. Figure \ref{fig:fig5} shows that although the proposed method uses temporal information from previous frames, it can still accurately estimate the fast motion field. Such experimental results verify the robustness of the proposed method in fast motion cases.

~\\
\textbf{Loss landscape.} 
To further analyze the optimization difficulty of the neural scene flow estimation, we demonstrate the loss landscape of FNSF and the proposed method in Figure \ref{fig:fig4}. It is well known that the high flatness of the minima indicates good generalization ability \cite{li2018visualizing, keskar2016large, ma2021rethinking, chen2023saks}. Figure \ref{fig:fig4} shows that the minima of the proposed method are more flat than FNSF. Therefore, the proposed method eases the scene flow optimization process and has better generalization ability, which also verifies the correctness of Theorem \ref{thmss2}.


\section{Conclusion}
In this paper, we theoretically analyze NSFP's generalization ability, finding that its generalization error decreases with more point clouds. Based on such theoretical findings, we can explain its effectiveness for large-scale point cloud scene flow estimation. Inspired by the theoretical findings, we propose a simple and effective multi-frame scene flow estimation scheme, which is dedicated to large-scale OOD autonomous driving scenarios. More crucially, we theoretically analyze the generalization ability of the proposed multi-frame method. The generalization error of the proposed method is bounded, indicating that adding multiple frames to the optimization process does not hurt its generalizability. Meanwhile, the case study demonstrates that such a multi-frame scheme eases the optimization process and can estimate fast motion fields. Both theoretical analysis and experimental results show the superiority of the proposed method in large-scale OOD autonomous driving applications.

\ifCLASSOPTIONcaptionsoff
  \newpage
\fi

\bibliographystyle{IEEE}
\bibliography{main.bbl}

\end{document}